\documentclass[twoside,11pt]{article}
\usepackage{url}
\usepackage{amsmath}
\usepackage{bm}
\usepackage{algorithm}
\usepackage{algpseudocode}
\usepackage{subcaption}
\usepackage{makecell}
\usepackage{multirow}
\usepackage{rotating}

\usepackage{appendix}
\usepackage{amsfonts}
\usepackage{amsthm}
\usepackage{natbib}

\usepackage{arxiv}
\usepackage{hyperref}

\newcounter{algsubstate}
\renewcommand{\thealgsubstate}{\alph{algsubstate}}
\newenvironment{algsubstates}
  {\setcounter{algsubstate}{0}%
   \renewcommand{\State}{%
     \stepcounter{algsubstate}%
     \Statex {\small\thealgsubstate:}\space}}
  {}

\newcommand{\softthresha}[2]{\textrm{sign}(#1)\cdot(|#1|-#2)^+}
\newcommand{\softthreshb}[2]{\textrm{sign}(#1)\cdot#2}
\newcommand{\Wgl}{\bm{W_{\textrm{PL}}}}
\newcommand{\Wglhat}{\bm{\hat{W}_{\textrm{PL}}}}
\newcommand{\R}[0]{\mathbb{R}}
\newcommand{\Rr}[0]{\left(\mathbb{R}_{\geq 0}\right)}

\newtheorem{proposition}{Proposition}
\newtheorem{definition}{Definition}
\newtheorem{lemma}{Lemma}

\begin{document}
\title{Non-linear, Sparse Dimensionality Reduction via Path Lasso Penalized Autoencoders}

\author{
       Oskar Allerbo\\
       \texttt{allerbo@chalmers.se}\\
       Mathematical Sciences\\
       University of Gothenburg and Chalmers University of Technology\\
       SE-412 96 Gothenburg, Sweden\\
       \And
       Rebecka J\"ornsten\\
       \texttt{jornsten@chalmers.se} \\
       Mathematical Sciences\\
       University of Gothenburg and Chalmers University of Technology\\
       SE-412 96 Gothenburg, Sweden\\
}

\maketitle

\begin{abstract}%
High-dimensional data sets are often analyzed and explored via the construction of a latent low-dimensional space which enables convenient visualization and efficient predictive modeling or clustering. For complex data structures, linear dimensionality reduction techniques like PCA may not be sufficiently flexible to enable low-dimensional representation. Non-linear dimension reduction techniques, like kernel PCA and autoencoders, suffer from loss of interpretability since each latent variable is dependent of all input dimensions. To address this limitation, we here present path lasso penalized autoencoders. This structured regularization enhances interpretability by penalizing each \emph{path} through the encoder from an input to a latent variable, thus restricting how many input variables are represented in each latent dimension. Our algorithm uses a group lasso penalty and non-negative matrix factorization to construct a sparse, non-linear latent representation. We compare the path lasso regularized autoencoder to PCA, sparse PCA, autoencoders and sparse autoencoders on real and simulated data sets. We show that the algorithm exhibits much lower reconstruction errors than sparse PCA and parameter-wise lasso regularized autoencoders for low-dimensional representations. Moreover, path lasso representations provide a more accurate reconstruction match, i.e.\ preserved relative distance between objects in the original and reconstructed spaces. 
\end{abstract}


\section{Introduction}
Dimensionality reduction is a key component in data compression, data visualization and feature extraction.
One of the most widely used techniques is principal component analysis (PCA), that uses the eigendecomposition of the sample covariance matrix to construct latent dimensions as linear combinations of the original dimensions. The interpretability of the latent representation is increased if each latent dimension consists only of a subset of the original dimensions, as in sparse PCA \citep{zou2006sparse}. Since the introduction of sparse PCA, several variants have been presented, such as non-negative sparse PCA \citep{zass2007nonnegative}, where the loadings are all non-negative; multilinear sparse PCA \citep{lai2014multilinear}, that operates on tensors instead of vectors; and robust PCA \citep{meng2012improve,croux2013robust}, that is less affected by outliers.

While more interpretable than standard PCA, sparse PCA and its variants are still linear and therefore cannot capture more complex relations in the data. Furthermore, they have limitations on how efficiently they make use of the latent dimensions, something that is especially important when the number of latent dimensions is small.
Several non-linear generalizations of PCA exist, the most important ones being kernel PCA \citep{scholkopf1998nonlinear} and autoencoders \citep{kramer1991nonlinear}. 
In kernel PCA, a kernel function is used to implicitly and non-linearly map the data to a space with higher dimensionality than $d_x$ (the original dimensionality) in which linear PCA is performed. Autoencoders use an hourglass shaped neural network with $d_x$ input and output nodes and $d_z$ (the latent dimensionality) nodes in the middle layer. The same data is used both as input and output, so the goal of the autoencoder is to reconstruct the original data from a lower dimensional representation, that is found in the middle layer.

Although there do exist several algorithms for sparse kernel PCA, see work by \citet{tipping2001sparse}, \citet{smola2002sparse}, \citet{wang2016sparse} and \citet{guo2019sparse}, as well as for sparse autoencoders (e.g.\ \citealp{ng2011sparse}), the terminology might be a little confusing, since the algorithms are sparse in a different sense than sparse PCA. Instead of being sparse in the sense  original to latent dimensions, they are sparse in the sense observations to latent dimensions, which means that each observation is active only in a subset of the latent dimensions (and vice versa, each latent dimension depends only on a subset of the data). To further illustrate this, we look at the linear case, with $d_x$ and $d_z$ as before and with $n$ being the number of observations. For $\bm{Z}\in \R^{n\times d_z}$, $\bm{X}\in \R^{n\times d_x}$ and $\bm{W} \in \R^{d_x\times d_z}$, the latent representation $\bm{Z}$ is obtained from the original representation $\bm{X}$ as 
\begin{equation}
\bm{Z}=\bm{X}\cdot \bm{W}.
\label{eq:ZXW}
\end{equation}
Then sparsity in the sense original to latent dimensions means that $\bm{W}$ is sparse, while sparsity in the sense observations to latent dimensions means that $\bm{Z}$ is sparse.
A third notion of sparsity is feature selection, where only a subset of the original dimension are included in the model. This corresponds to entire rows of $\bm{W}$ being zero, allowing the remaining rows to be dense. 

The sparse autoencoder generalizes Equation \eqref{eq:ZXW} to $\bm{Z}=\textrm{Enc}(\bm{X})$, where Enc is the non-linear encoder and, to produce a sparse $\bm{Z}$, it includes a constraint on how frequently a latent unit is allowed to be non-zero, where frequency is measured among observations. This means that the architecture of the encoder might still be dense, in contrast to sparsity corresponding to a sparse $\bm{W}$, which requires a sparse architecture.

Neural networks with sparse architectures are usually obtained using different versions of lasso, or $l_1$-, regularization \citep{tibshirani1996regression}, and there are numerous examples of these. \citet{scardapane2017group} used group lasso \citep{yuan2006model} to remove all the links to or from a node, while \citet{yoon2017combined} combined group lasso with exclusive lasso \citep{zhou2010exclusive} on filters in convolutional neural networks to, on one hand, totally eliminate some filters (using group lasso) and, on the other hand, make filters as different as possible (using exclusive lasso). Lasso regularized autoencoders include work by \citet{wu2019learning}, with a linear encoder and a lasso regularized decoder; \citet{dabin2020blind}, with a lasso regularized encoder and a linear decoder; and \citet{ainsworth2018oi}, which uses group lasso and variational autoencoders to split the original dimensions into pre-defined groups and then uses one decoder per group, with a shared latent space.

In this paper we propose path lasso, that uses group lasso regularization to eliminate all connections between two nodes in two non-adjacent layers of a fully connected feedforward neural network. We apply path lasso, in combination with exclusive lasso, to an autoencoder to introduce sparsity between original and latent variables, obtaining non-linear, sparse dimensionality reduction in the same sense as in sparse PCA. Path lasso forces each latent dimension to be a function only of a subset of the original dimensions, while exclusive lasso encourages these subsets to differ.
To the best of our knowledge this is the first non-linear dimensionality reduction algorithm that is sparse in this sense.

The rest of this paper is organized as follows: In Section \ref{sec:method} we introduce the path lasso penalty and the path lasso penalized autoencoder. In Section \ref{sec:experiments} we run experiments on real and simulated data sets, comparing path lasso to PCA, autoencoders, sparse autoencoders, sparse PCA and an autoencoder with parameter-wise $l_1$-regularization. We show that, for a given sparsity, path lasso results in a lower reconstruction error and is better at reconstruction match, i.e.\ retaining relative positioning of objects in the reconstructed space as in the original space. We conclude with a discussion in Section \ref{sec:conculsions}.

\section{Method}
\label{sec:method}
This section is structured in the following way: Sections \ref{sec:lassos} and \ref{sec:prox} present short reviews of different flavours of the lasso algorithm and of proximal gradient descent, while Sections \ref{sec:paths} to \ref{sec:path_ae} describe different aspects of path lasso. Section \ref{sec:paths} and \ref{sec:paths_to_links} describe how paths between nodes in two non-adjacent layers are defined and penalized, and how the path penalties are transformed to individual link penalties. Section \ref{sec:apply_pen} discusses when and how to apply the path penalty and Section \ref{sec:path_ae} describes how we adapt path lasso when using it in an autoencoder. In Appendix \ref{sec:speed} we discuss different methods to accelerate training.

\subsection{Review of Lasso Penalties}
\label{sec:lassos}
The lasso algorithm sets some model parameters exactly equal to zero, thus eliminating them from the model. There are different versions of the lasso, four of which are used in this paper. Here follows a short summary to these.

\emph{(Standard) Lasso} applies an $l_1$-penalty to all the parameters in the parameter vector $\bm{\theta}\in \R^d$, and is defined as
$$\lambda\|\bm{\theta}\|_1 = \lambda\sum_{i=1}^d |\theta_i|,$$
where $\lambda>0$ is the regularization strength. Due to the non-differentiability of the absolute value at zero, some of the $\theta_i$'s are set to exactly zero and thus eliminated from the model.

\emph{Adaptive Lasso} applies an individual $l_1$-penalty to all the parameters in the parameter vector $\bm{\theta}\in \R^d$, and is defined as
$$\sum_{i=1}^d\lambda_i|\theta_i|,$$
where $\lambda_i := \frac \lambda {|\hat{\theta}^{\textrm{R}}_i|^\gamma}$ for some $\gamma >0$ (common practice is $\gamma=2$), and $\hat{\theta}^\textrm{R}_i$ is the ridge regression estimate of $\theta_i$. The idea is that important parameters will have larger values of $\hat{\theta}^{\textrm{R}}_i$ and thus be penalized less than unimportant parameters.

\emph{Group Lasso} penalizes pre-defined groups of parameters together, which means that either all, or none, of the parameters in the group are set to zero. The group lasso penalty for a group $g \in \mathcal{G}$ is defined as 
$$\lambda\|\bm{\theta_{g}}\|_2=\lambda\sqrt{\sum_{i\in g}\theta_i^2},$$
where $\mathcal{G}=\{g_1,\dots g_G\}$ is a disjoint partition of the index set $\{1,\dots,d\}$, i.e.\ each $g$ is a set of indices defining a group, and, for a given $\bm{\theta} \in \R^d$, $\bm{\theta_{g}}$ is a $d$-dimensional vector with components equal to $\bm{\theta}$ for indices within $g$ and zero otherwise. The total group lasso penalty is then taken as the sum of the penalties over the different groups. As seen, for a given group $g$,
$$\sqrt{\sum_{i\in g}\theta_i^2} =0 \iff \theta_i=0\ \forall i \in g.$$

The \emph{Exclusive Lasso} can be seen as the opposite of the group lasso. Again, the parameters are split into pre-defined groups, but now the goal is instead to impose a similar number of non-zero parameters in every group. With $g$ as before, its exclusive lasso penalty defined as 
$$\lambda\|\bm{\theta_{g}}\|_1^2=\lambda\left(\sum_{i\in {g}}|\theta_i|\right)^2,$$
and, again, the total penalty is defined as the sum over the groups. Since the number of mixed terms in the squared sum grows with the number of elements in the sum, the total number of mixed terms over all the groups is minimized when the non-zero elements are evenly distributed among the groups.

\subsection{Review of Proximal Gradient Descent}
\label{sec:prox}
The reason that some parameters in lasso penalized models are set exactly to zero, is that the derivative of the absolute value is not unique at zero: $\left.\frac{\partial |\theta|}{\partial \theta}\right|_{\theta=0}\in [-1,1]$. Gradient descent methods are unable to use this non-uniqueness and as a consequence no parameters are set exactly to zero. Proximal gradient descent \citep{rockafellar1976monotone} on the other hand takes the non-uniqueness into account, resulting in exact zeros.

If the objective function, $f(\bm{\theta})$, can be decomposed into $f(\bm{\theta}) = g(\bm{\theta}) + h(\bm{\theta})$, where $g$ is differentiable (typically a reconstruction error) and $h$ is not (typically a lasso regularization term), then a standard gradient descent step, followed by a proximal gradient descent step, is defined as
\begin{equation*}
\bm{\theta}^{t+1} = \textrm{prox}_{\alpha h}(\bm{\theta}^t-\alpha\nabla g(\bm{\theta}^t)),
\label{eq:prox_grad_desc}
\end{equation*}
where $\alpha>0$ is the learning rate and prox is the proximal operator that depends on $h$. 

For lasso, with $h(\bm{\theta}) = \lambda||\bm{\theta}||_1 = \sum_i \lambda |\theta_i|$, the proximal operator decomposes component-wise and is, with 
$(x)^+:=\textrm{max}(x, 0)$,
\begin{equation}
\label{eq:lasso_prox}
\textrm{prox}_{\alpha h}(\theta_i) = \textrm{sign}(\theta_i)\cdot (|\theta_i| - \alpha \lambda)^+.
\end{equation}

I.e., each $\theta_i$ is additively shrunk towards zero, and once it changes sign it becomes exactly zero.

For group lasso, with $h(\bm{\theta}) = \lambda\sum_{g \in \mathcal{G}}||\bm{\theta_g}||_2 = \lambda\sum_{g \in \mathcal{G}}\sqrt{\sum_{i \in g}\theta_i^2}$, where $\bm{\theta_g}$ are the $\theta_i$'s that belong to group $g$ and $\mathcal{G}$ is the set of all groups, the proximal operator for $\theta_i$ is
\begin{equation}
\label{eq:gl_prox}
\textrm{prox}_{\alpha h}(\theta_i) = \theta_i\cdot \left(1-\frac{\alpha \lambda}{\sqrt{\sum_{j\in g_i} \theta_j^2}}\right)^+,
\end{equation}
where $g_i$ is the group that $\theta_i$ belongs to. Thus all members of the group are penalized equally and set to zero at the exact same time.

\subsection{Path Penalties}
\label{sec:paths}
Let $\{\bm{o_l}\}_{l=0}^L,\ \bm{o_l}\in \R^{d_l},$ denote the outputs of $L+1$ consecutive layers in a fully connected feedforward neural network, where the first and last layers are also denoted by $\bm{x}$ and $\bm{y}$ respectively, i.e.\ $\bm{x}:=\bm{o_0}$ and $\bm{y}:=\bm{o_L}$. Then $\bm{y}$ depends on $\bm{x}$ as
\begin{equation}
\label{eq:nn}
\bm{y} = \Phi_L(\bm{W_{L}}\Phi_{L-1}(\bm{W_{L-1}}\Phi_{L-2}(\dots\Phi_1(\bm{W_1}\bm{x}+\bm{b_1})\dots) +\bm{b_{L-1}})+\bm{b_{L}}),
\end{equation}
where $\{\bm{W_l}\}_{l=1}^L,\ \bm{W_l}\in \R^{d_l\times d_{l-1}},$ are the weight matrices, $\{\bm{b_l}\}_{l=1}^L,\ \bm{b_l}\in \R^{d_l},$ the bias vectors, and $\{\Phi_l\}_{l=1}^L$ the (not necessarily identical) element-wise activation functions.
Thinking of Equation \eqref{eq:nn} as a graph, each weight matrix element, $w^l_{i_1i_2}:=(\bm{W_l})_{i_1i_2}$, corresponds to a link between two nodes in two adjacent layers

 By combining links from multiple weight matrices, paths between nodes in non-adjacent layers can be constructed. Between a given node in layer $\bm{x}$, $x_{i_0}$, and a node in layer $\bm{y}$, $y_{i_L}$, there are in total $\prod_{l=1}^{L-1}d_l$ paths, each consisting of $L$ links, see Figure \ref{fig:two_paths} for an illustration. A path is broken if at least one of its links has value zero and to disconnect the two nodes, all paths between them need to be broken. Just as \citet{neyshabur2015norm}, we define the value of a path as the product of its absolute valued links; see Definition \ref{def:path}.
\begin{definition}[Path Value]
\label{def:path}
For $k=1,2,\dots,\prod_{l=1}^{L-1}d_k$, where each $k$ corresponds to a unique combination of the indices $i_1$ to $i_{L-1}$
$$p^k_{i_Li_0}:= |w^L_{i_Li^k_{L-1}}|\cdot |w^{L-1}_{i^k_{L-1}i^k_{L-2}}|\cdot \dots |w^1_{i^k_1i_0}|.$$
\end{definition}
With this definition a broken path, where at least one link is zero, has the value zero. We further define a group of paths so that all paths connecting $x_{i_0}$ to $y_{i_L}$ form one group; then if all paths in the group are zero, the nodes are disconnected. 
Applying the group lasso proximal operator to path $p^k_{i_Li_0}$, belonging to group $g_{i_Li_0}$ then amounts to
\begin{equation}
\label{eq:prox_path}
p^k_{i_Li_0}\cdot\left(1-\frac{\alpha\lambda}{\sqrt{\sum_{l\in g_{i_Li_0}}(p^l_{i_Li_0})^2}}\right)^+, 
\end{equation}
which according to Proposition \ref{thm:GL_eq_paths} can be written as
$$p^k_{i_Li_0}\cdot\left(1-\frac{\alpha\lambda}{(\Wgl)_{i_Li_0}}\right)^+$$
with $\Wgl$ according to Definition \ref{def:wgl}.

\begin{definition}[Path Lasso Matrix]
\label{def:wgl}
With the square and the square root taken element-wise,
\begin{equation}
\label{eq:group_lasso}
\Wgl:=\sqrt{\prod_{l=L,\dots,1} (\bm{W_{l}})^2}=\sqrt{(\bm{W_{L}})^2\cdot (\bm{W_{L-1}})^2\cdot \dots (\bm{W_{1}})^2}.
\end{equation}
\end{definition}

\begin{proposition}\label{thm:GL_eq_paths}
The element $(i_L, i_0)$ in $\Wgl$ is the group lasso penalty on the group consisting of all paths between nodes $x_{i_0}$ and $y_{i_L}$, i.e.
\begin{equation*}
(\Wgl)_{i_Li_0}=\sqrt{\sum_{k\in g_{i_Li_0}} (p^k_{i_Li_0})^2}
\end{equation*}
\end{proposition}
For a proof, see Appendix \ref{sec:proofs}.

\begin{figure}
  \includegraphics[width=1.\textwidth]{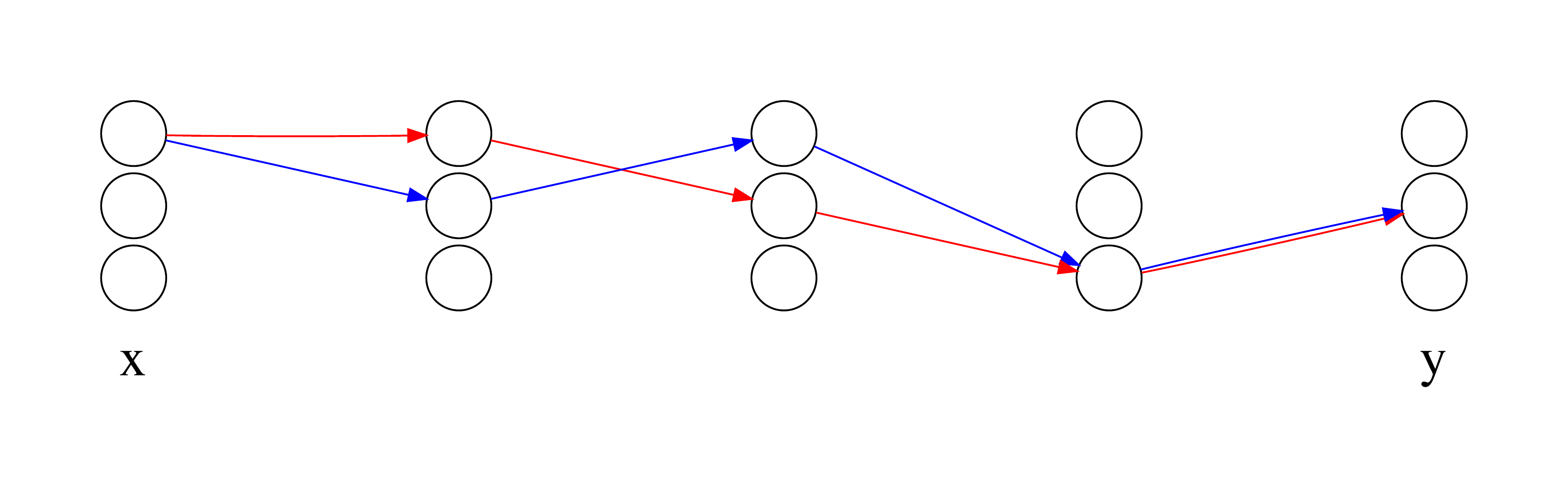}
  \caption{Illustration of two of the $3^3$ (where the exponent is the number of inner layers and the base is their width) possible paths between node one in layer one and node two in layer five. Each arrow denotes a link, which corresponds to an element in a weight matrix.}
  \label{fig:two_paths}
\end{figure}

We call each element in $\Wgl$ a connection, i.e.\ $(\Wgl)_{i_Li_0}$ is the strength of the connection between nodes $x_{i_0}$ and $y_{i_L}$ and if it is zero, meaning that all paths between the nodes are broken, then the nodes are disconnected. Proposition \ref{thm:GL_der} gives a theoretical justification for this definition of connections, stating that if there is no connection between $x_{i_0}$ and $y_{i_L}$, then the derivative of $y_{i_L}$ with respect to $x_{i_0}$ is zero, regardless of the value of $\bm{x}$.
\begin{proposition}\label{thm:GL_der}
Let the vector $\bm{y}$ depend on the vector $\bm{x}$ as stated in Equation \eqref{eq:nn} and let $\Wgl$ be the path lasso matrix defined in Equation \eqref{eq:group_lasso}. Then, if all weights and activations are bounded,
$$(\Wgl)_{i_Li_0}=0 \implies \frac{\partial y_{i_L}}{\partial x_{i_0}}=\left(\frac{\partial \bm{y}}{\partial \bm{x}}\right)_{i_Li_0} = 0 \qquad \forall \bm{x} \in \mathbb{R}^{d_0}.$$
\end{proposition}
For a proof, see Appendix \ref{sec:proofs}.

\subsection{From Paths to Links}
\label{sec:paths_to_links}
Equation \eqref{eq:prox_path} describes how all paths between two nodes in the network are penalized towards zero, but it does not tell how this penalization affects the individual links. However, since the neural network is expressed in terms of the individual links, rather than in paths, the penalized paths must be translated into penalized links, i.e.\ the penalty must be expressed on link level, rather than on path level. Proposition \ref{thm:paths_to_links} describes how this translation can be done, i.e.\ how applying the proximal group lasso operator to the paths can be translated into applying a standard lasso proximal operator to each individual link.
\begin{proposition}\label{thm:paths_to_links}
The group lasso proximal step in path space can be transformed to standard lasso proximal steps in link space by first solving the matrix equation
\begin{equation}
\prod_{l=L,\dots,1}|(\bm{W_l})^t|\odot\left(1-\frac{\alpha\lambda}{(\Wgl)^t}\right)^+ = \prod_{l=L,\dots,1}\bm{\tilde{W}_l}
\label{eq:eq_syst_mat}
\end{equation}
for $\{\bm{\tilde{W}_l}\}_{l=1}^L$, and then set
$(\bm{W_l})^{t+1}:=\emph{sign}((\bm{W_l})^t)\odot \bm{\tilde{W}_l}$, where $\tilde{w}^l_{ij} =:(|(w^l_{ij})^t|-\alpha\lambda^{l,t}_{ij})^+$ for some $\lambda^{l,t}_{ij}>0$.\\
The absolute value, sign and division are taken element-wise and $\odot$ denotes element-wise multiplication.
\end{proposition}
The proof, which is presented in Appendix \ref{sec:proofs}, contains a step where the paths in each group are summed over. This step transforms the system of non-linear equations from one equation per path to one equation per connection, i.e.\ from $\prod_{l=0}^{L}d_l$ to $d_0\cdot d_L$ equations, and to a form can be solved efficiently using non-negative matrix factorization, NMF, with the extra requirement that $\tilde{w}^l_{i_li_{l-1}}\leq|(w^l_{i_li_{l-1}})^t|$, or equivalently $\lambda^{l,t}_{i_li_{l-1}}\geq 0$. We describe this algorithm in Appendix \ref{sec:nmf}. 

The reduced number of equations leads to an undetermined system (unless the hidden layers are very narrow). Therefore, all equations can hold, although there might be more than one solution. Since we are interested in a solution that lies close to the unpenalized weight matrices, we use $|\bm{W_l}|$ as the seed for each $\bm{\tilde{W}_l}$ in the matrix factorization.

An optimization step using proximal gradient descent on paths is summarized in Algorithm \ref{alg:prox_path_step}. The bottleneck of this algorithm is the matrix factorization step from Equation \eqref{eq:eq_syst_mat}. In Appendix \ref{sec:speed}, different methods to alleviate this bottleneck are discussed.

\begin{algorithm}[ht]
\caption{Proximal Path Lasso Optimization Step}
\textbf{Input:} Parameters at time $t$, $\{(\bm{W_l})^t, (\bm{b_l})^t\}_{l=l}^L$; data, $\bm{x}$; learning rate, $\alpha$; regularization strength, $\lambda$.\\
\textbf{Output:} Parameters at time $t+1$: $\{(\bm{W_l})^{t+1}, (\bm{b_l})^{t+1}\}_{l=l}^L$.\\
\begin{algorithmic}[1]
  \State Update all weights and biases using one step of standard (stochastic) gradient descent:
  $$\{(\bm{W_l})^{t+\frac12},(\bm{b_l})^{t+\frac12}\}_{l=1}^L\leftarrow \{(\bm{W_l})^{t},(\bm{b_l})^{t}\}_{l=1}^L-\alpha\cdot\nabla\left(\textrm{NN}(\{(\bm{W_l})^{t},(\bm{b_l})^{t}\}_{l=1}^L;\bm{x})\right),$$
where NN denotes the neural network.
  \State Penalize paths, by applying the group lasso proximal operator, and update the path values accordingly:
  \begin{algsubstates}
  \State Construct the path lasso penalty matrix for the weight outputs from step 1, according to Equation \eqref{eq:group_lasso}:
  $$(\Wgl)^{t+\frac12}\leftarrow \sqrt{\prod_{l=L,\dots,1}\left((\bm{W_l})^{t+\frac{1}{2}}\right)^2}.$$
  \State Penalize paths according to Equation \eqref{eq:eq_syst_mat}:
  $$\bm{P}^{t+1}\leftarrow\prod_{l=L,\dots,1}|(\bm{W_l})^{t+\frac{1}{2}}|\odot\left(1-\frac{\alpha\lambda}{(\Wgl)^{t+\frac12}}\right)^+.$$
  \end{algsubstates}
  \State Translate penalized paths to penalized links:
  \begin{algsubstates}
  \State Translate penalized paths to absolute valued penalized links, using modified non-negative matrix factorization:
  $$\left\{(\bm{\tilde{W}_l})^{t+1}\right\}_{l=1}^L\leftarrow \textrm{NMF}\left(\bm{P}^{t+1}\right).$$
  \State Restore signs:
  $$(\bm{W_l})^{t+1}\leftarrow\textrm{sign}((\bm{W_l})^t)\odot (\bm{\tilde{W}_l})^{t+1}.$$
  \end{algsubstates}
\end{algorithmic}
\label{alg:prox_path_step}
\end{algorithm}

\subsection{Applying the Path Lasso Penalty}
\label{sec:apply_pen}
Since training a neural network is a non-convex optimization problem, with multiple local optima, how and when the regularization is added might affect which optimum is found. If a too high penalty is added too early during training, the risk of getting stuck in a bad local optimum is larger than if the regularization is added later. To mitigate this, training was split into three stages:
\begin{itemize}
\item
Following the adaptive lasso approach of \citet{allerbo2020flexible}, we first trained the network without  path regularization to obtain individual penalties for each connection during the path lasso stage. 
\item
In a second stage, we added path regularization with an individual penalty to each connection, depending on the magnitude of the connection after the first stage:
\begin{equation*}
\label{eq:adap_lbda}
\lambda_{i_L,i_0} := \frac{\lambda}{((\Wglhat)_{i_Li_0})^\gamma},
\end{equation*}
where $\Wglhat$ is the value of $\Wgl$ after the first optimization stage and $\gamma >0$. Throughout this paper, $\gamma=2$ was used.
\item
To reduce bias and thus improve performance, we finally added a stage of unregularized training after the path lasso stage, with the links set to zero in the previous stage kept to zero. 
\end{itemize}
All stages were trained until convergence and stages two and three were warm started with the solutions from the previous stage.

\subsection{Path Lasso for Dimensionality Reduction}
\label{sec:path_ae}
Path lasso as described so far is applicable to any two non-adjacent layers in any feedforward neural network. To use it for sparse non-linear dimensionality reduction, it was applied to an autoencoder, with the following adaptations:

\begin{itemize}
\item
Path penalties were applied between the input ($\bm{x}\in \R^{d_x}$) and latent ($\bm{z}\in \R^{d_z}$) variables, and between the latent and output ($\bm{\hat{x}}\in\R^{d_x}$) variables. 
\item
To enforce the encoder and the decoder to be symmetric, the group lasso groups were defined as all paths connecting $x_i$ and $z_j$, together with all paths connecting $z_j$ and $\hat{x}_i$, i.e.
$$\Wgl:=\sqrt{\prod_{l=L,\dots,1} (\bm{W^E_{l}})^2+\left(\prod_{l=L,\dots,1} (\bm{W^D_{l}})^2\right)^\top},$$
where $\{\bm{W^E_l}\}_{l=1}^L$ and $\{\bm{W^D_l}\}_{l=1}^L$ are the weight matrices of the encoder and the decoder, respectively. This means that if an input is disconnected to a latent variable, so is the corresponding output, by construction.
\item
To encourage the algorithm to make equal use of the latent dimensions, we added an exclusive lasso penalty to the elements in $\Wgl$, with as many groups as there are latent dimensions, each group being defined as the connections to a given latent dimension. 
\end{itemize}

\section{Experiments}
\label{sec:experiments}
In order to evaluate path lasso for dimensionality reduction we applied it to three different data sets, one with synthetic data, consisting of Gaussian clusters on a hypercube, one with text documents from newsgroup posts, and one with images of faces. In each experiment, 20 \% of the data was set aside for testing and the remaining 80 \% was split 90-10 into training and validation data; all visualizations were made using the testing data.
All autoencoders used one hidden layer with tanh activations in the encoder and decoder respectively, and were trained with $l_2$-loss. For optimization stages not using proximal gradient descent, the Adam optimizer \citep{kingma2014adam} was used.

In addition to path lasso, we used a standard autoencoder, a sparse autoencoder, an autoencoder with parameter-wise $l_1$-regularization and thresholding (hereafter referred to as standard lasso), PCA and sparse PCA. The reason for adding thresholding in standard lasso is because unless proximal methods are used, no parameters are set exactly to zero, but instead to very small values, as discussed in Section \ref{sec:prox}. It should also be noted that the sparse autoencoder is sparse in terms of observations to latent dimensions, and not in terms of original to latent dimensions, as we are interested in; see the text associated to Equation \eqref{eq:ZXW} for details. The standard autoencoder and PCA are of course not sparse in any sense. Since "observation sparse" algorithms are not as relevant to us as the "truly sparse" algorithms, we omit them in some of the comparisons.

The following measures were calculated on the testing data: 
\begin{itemize}
\item
Reconstruction error as explained variance ($R^2$).
\item
Fraction of correctly identified reconstructions. A reconstructed observation is considered correctly identified if it is closer to its own original observation than any of the other ones, measured in $l_2$-distance, i.e.\ $\|\bm{\hat{x}_i}-\bm{x_i}\|_2<\|\bm{\hat{x}_i}-\bm{x_j}\|_2$, $i\neq j$. This is hereafter referred to as observation reconstruction match.
\item
Fraction of correctly reconstructed labels. The reconstructed label is defined as the label of the original observation that is closest to the reconstructed observation, where distance is measured as above, i.e.\ the label of $\bm{x_j}$, where $\|\bm{\hat{x}_i}-\bm{x_j}\|_2<\|\bm{\hat{x}_i}-\bm{x_k}\|_2$, $j\neq k$. This is hereafter referred to as label reconstruction match.
\end{itemize}

\subsection{Synthetic Data Set}
\label{sec:synth}
Sixteen clusters were generated in $\mathbb{R}^4$, centered at each of the sixteen vertices in the hypercube $\{0,1\}^4$. For cluster $i$, 100 data points were sampled according to $\bm{x_i} \sim \mathcal{N}(\bm{\mu_i},0.01\cdot \bm{I_4})$, where $\bm{I_4}$ is the identity matrix and $\bm{\mu_i}$ is one of the sixteen vertices in $\{0,1\}^4$. The four dimensional data set was reduced down to two dimensions using the six different algorithms. For the four autoencoder based algorithms, the number of nodes in the five layers of the autoencoder were 4, 50, 2, 50 and 4, respectively. The three sparse algorithms (in the sense original to latent dimensions) were penalized so that four of the original eight connections remained. The experiment was performed twice, with and without added noise, distributed according to $\mathcal{N}(0,0.3^2)$.

Ten different splits of the data into training and validation sets were done. For each split, three different optimization seeds were used and the seed resulting in the best $R^2$ value on the validation data was chosen. The resulting mean and standard deviations are presented in Table \ref{tab:gauss}. The p-values come from the one-sided paired rank test, testing whether path lasso performs better than the competing algorithm. P-values smaller than 1 \% are marked in bold. With noise added, path lasso performs significantly better than the other algorithms both in terms of $R^2$ and reconstruction match. Without noise path lasso still performs better than the dense algorithms in terms of observation reconstruction match, while all the four non-linear methods perform very well in terms of $R^2$ and label reconstruction.

The results with the best $R^2$ values are plotted in Figure \ref{fig:gauss}, where clusters that are diagonal to each other in $\R^4$ (e.g.\ $(0,1,0,0)$ and $(1,0,1,1)$) are plotted using the same color, but with different markers - circles or crosses. For the three sparse algorithms, path lasso, standard lasso and sparse PCA, each of the two latent dimensions becomes a combination of two of the original four dimensions, which can be seen in the axis aligned data in the plots. Even with no added noise the linear algorithms, PCA and sparse PCA, are not able to fully separate the sixteen clusters, while all four non-linear algorithms, based an autoencoders, are.

\begin{sidewaystable}
\centering
\begin{tabular}{|l|l|l|l|l|l|l|l|l|}
 \hline
\multirow{3}{*}{Algorithm} & \multirow{3}{*}{Noise} & \multirow{3}{*}{Connections}
& \multicolumn{2}{c|}{\multirow{2}{*}{$R^2$}} & \multicolumn{4}{c|}{Reconstruction Match}\\
\cline{6-9}
& & & \multicolumn{2}{c|}{ } & \multicolumn{2}{c|}{Observation} & \multicolumn{2}{c|}{Label} \\
\cline{4-9}
& & & Mean (std) & p-value & Mean (std) & p-value & Mean (std) & p-value\\
\hline
\hline
Path Lasso & No & $4$ & $0.98\ (0.0013)$ & - & $0.37\ (0.015)$ & - & $1.0\ (0.0014)$ & -\\
Standard Lasso & No & $4$ & $0.98\ (0.0020)$ & $0.28$ & $0.36\ (0.022)$ & $0.070$ & $1.0\ (0.0014)$ & $0.68$\\
Sparse PCA & No & $4$ & $0.48\ (0.0018)$ & $\bm{0.00098}$ & $0.067\ (0.019)$ & $\bm{0.00098}$ & $0.30\ (0.014)$ & $\bm{0.0029}$\\
Autoencoder & No & $8$ & $0.98\ (0.0026)$ & $0.019$ & $0.35\ (0.013)$ & $\bm{0.0046}$ & $1.0\ (0.0021)$ & $0.16$\\
Sparse AE & No & $8$ & $0.98\ (0.00027)$ & $0.50$ & $0.34\ (0.0076)$ & $\bm{0.0039}$ & $1.0\ (0.0013)$ & $0.018$\\
PCA & No & $8$ & $0.48\ (0.0014)$ & $\bm{0.00098}$ & $0.067\ (0.015)$ & $\bm{0.00098}$ & $0.43\ (0.044)$ & $\bm{0.0029}$\\
\hline
Path Lasso & Yes & $4$ & $0.89\ (0.017)$ & - & $0.38\ (0.038)$ & - & $0.88\ (0.019)$ & -\\
Standard Lasso & Yes & $4$ & $0.58\ (0.097)$ & $\bm{0.00098}$ & $0.11\ (0.030)$ & $\bm{0.00098}$ & $0.39\ (0.11)$ & $\bm{0.0029}$\\
Sparse PCA & Yes & $4$ & $0.50\ (0.0016)$ & $\bm{0.00098}$ & $0.11\ (0.0061)$ & $\bm{0.0029}$ & $0.37\ (0.010)$ & $\bm{0.0029}$\\
Autoencoder & Yes & $8$ & $0.86\ (0.011)$ & $\bm{0.0029}$ & $0.33\ (0.014)$ & $\bm{0.0095}$ & $0.86\ (0.016)$ & $\bm{0.0088}$\\
Sparse AE & Yes & $8$ & $0.85\ (0.0038)$ & $\bm{0.0020}$ & $0.30\ (0.015)$ & $\bm{0.00098}$ & $0.85\ (0.0035)$ & $\bm{0.0020}$\\
PCA & Yes & $8$ & $0.50\ (0.0043)$ & $\bm{0.00098}$ & $0.12\ (0.0088)$ & $\bm{0.00098}$ & $0.40\ (0.030)$ & $\bm{0.00098}$\\
\hline
\end{tabular}
\caption{Number of remaining connections, explained variance, reconstruction match and p-value for six the algorithms when reducing 16 clusters from four to two dimensions with and without added noise. P-values are for the one-sided paired rank test, that tests whether path lasso performs better than the competing algorithm, with p-values smaller than 1 \% marked in bold. For the three sparse algorithms, the number of connections is always four, by construction; for three the dense algorithms it is always eight. Especially for noisy data, path lasso outperforms the competing algorithms.}
\label{tab:gauss}
\end{sidewaystable}

\begin{figure}
  \centering
  \includegraphics[width=1\textwidth]{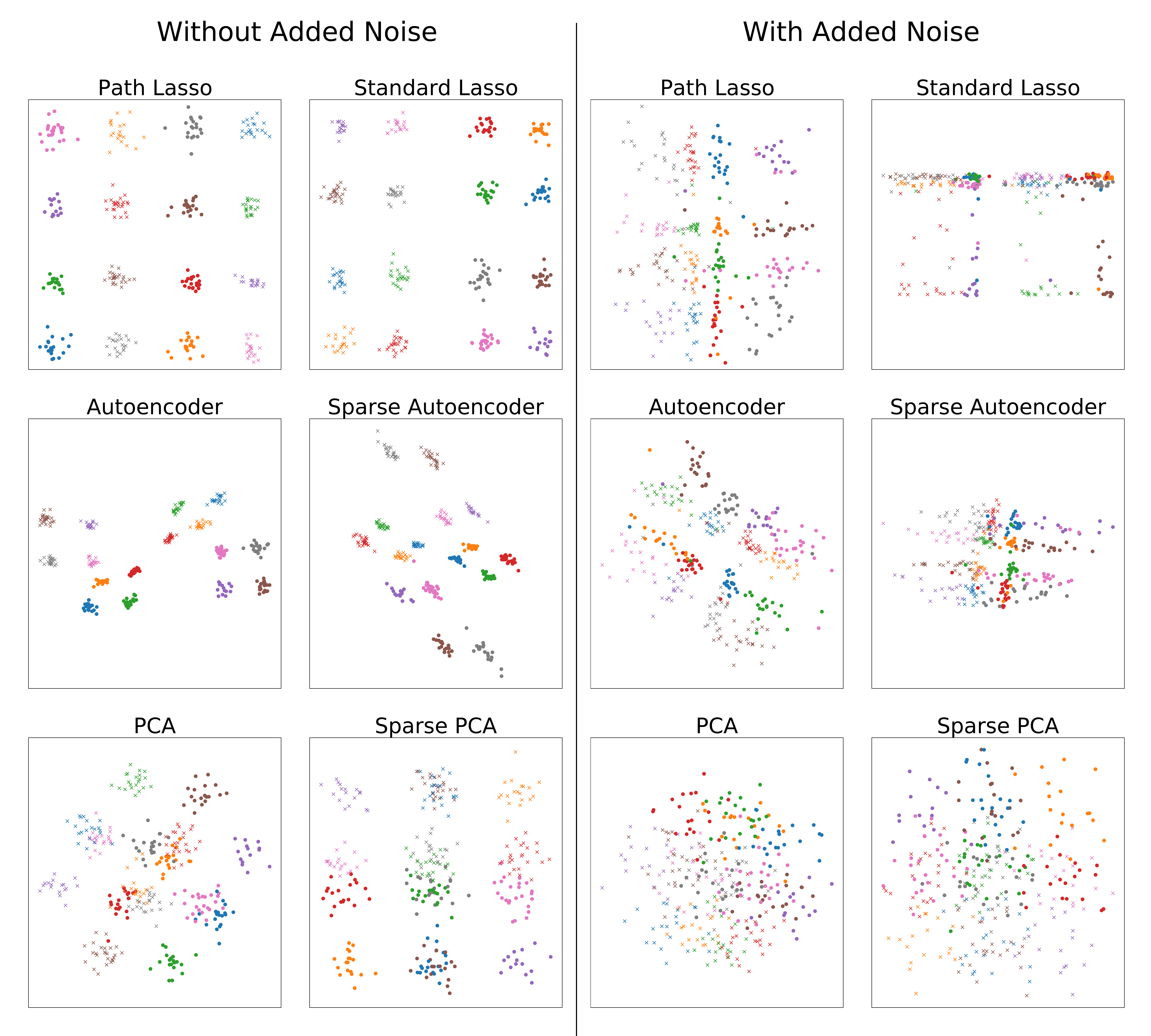}
  \caption{Reduction of 16 clusters from four to two dimension with and without added noise, using six different algorithms. Clusters that are diagonal to each other in $\R^4$ have the same color, but different markers.}
  \label{fig:gauss}
\end{figure}

\subsection{Text - 20 Newsgroup Data Set}
To test the algorithm on text data, the 20 newsgroups data set\footnote{Available at \url{http://qwone.com/~jason/20Newsgroups/}.} was used. Out of the original 20 categories, the following 4 were selected: \texttt{soc.religion.christian}, \texttt{sci.space}, \texttt{comp.windows.x} and \texttt{rec.sport.hockey}, which resulted in 31225 documents. Then, for each word the tf-idf score was calculated, and the 100 words with the highest score were kept, resulting in a $31225\times 100$ data matrix. Using a standard autoencoder, path lasso, standard lasso and sparse PCA, the 100 dimensional data set was mapped down to two, four and 25 dimensions, where the 25D case was added for comparing the test measures at moderate dimensionality reductions. For the autoencoder based algorithms, the layer widths were 100, 50, 2 (4, 25), 50 and 100 nodes. 

Two different sparsity levels were used; One sparse, to be able to compare interpretability, and one almost dense, to be able to compare to a standard autoencoder. In each case the penalties were set to obtain (approximately) the same sparsity, measured as number of connections, for all algorithms.

The latent spaces in the sparse case are shown in Figure \ref{fig:news}. In the 4D case, standard lasso only uses one latent dimension and is incapable of distinguishing between the categories, sparse PCA maps one category to each latent dimension, while path lasso has a tendency to map two categories to each latent dimension, one to the positive and one to the negative axis. This is further accentuated when compressing to only two dimensions. While path lasso is able to identify all four categories, sparse PCA only identifies two of them. Standard lasso still only uses one dimension and is not able to identify any categories at all. The use of only one latent dimension by standard lasso is likely attributed to the flexibility of the autoencoder, and without the structure in the penalty imposed by the path lasso algorithm, there is less incentive to use all parts of the latent space. The same tendency is visible in Figure \ref{fig:att_eig1}.

To see which original dimensions (words) contribute to which latent dimensions, the non-zero elements of $\Wgl$ can be used, but since all elements in this matrix are non-negative, it gives no information about the sign. Instead the corresponding signed matrix was created according to $\bm{W_2}\cdot \bm{W_1}+(\bm{W_4}\cdot \bm{W_3})^\top$, where $(\bm{W_1}, \bm{W_2})$ and $(\bm{W_3}, \bm{W_4})$ are the weight matrices of the encoder and the decoder, respectively. The signed words in the latent dimensions are presented in Tables \ref{tab:news_words2} and \ref{tab:news_words4}. The assignment of words to the 4 (8) latent half-axes done by path lasso and sparse PCA is consistent with the results in Figure \ref{fig:news}. Path lasso also seems to identify a subcategory of \texttt{comp.windows.x} related to e-mails.

In Tables \ref{tab:news_stats_high} and \ref{tab:news_stats_low}, remaining connections, explained variance and reconstruction match are presented for both sparsity levels. We conclude that path lasso performs best among the sparse algorithms both in terms of explained variance and reconstruction match. Compared to the standard autoencoder, path lasso performs slightly better in terms of reconstruction error and observation reconstruction match, which is in line with the results in Section \ref{sec:synth}. It is also noticeable that with increasing latent dimensionality the advantage of non-linear over linear models decreases.

\begin{figure*}
     \centering
     \begin{subfigure}[b]{\textwidth}
         \centering
         \includegraphics[width=0.94\textwidth]{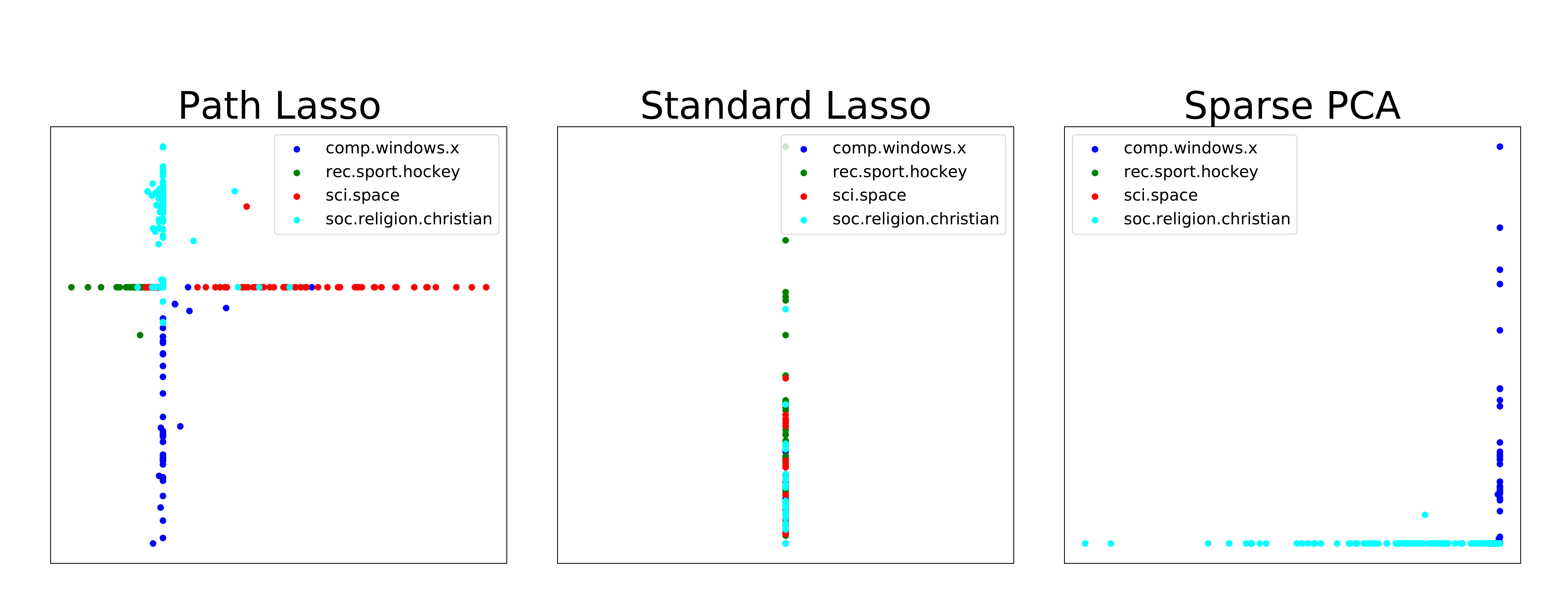}
         \caption{Latent 2D space for the newsgroup data}
         \label{fig:news2}
     \end{subfigure}
     \begin{subfigure}[b]{\textwidth}
         \centering
         \includegraphics[width=0.92\textwidth]{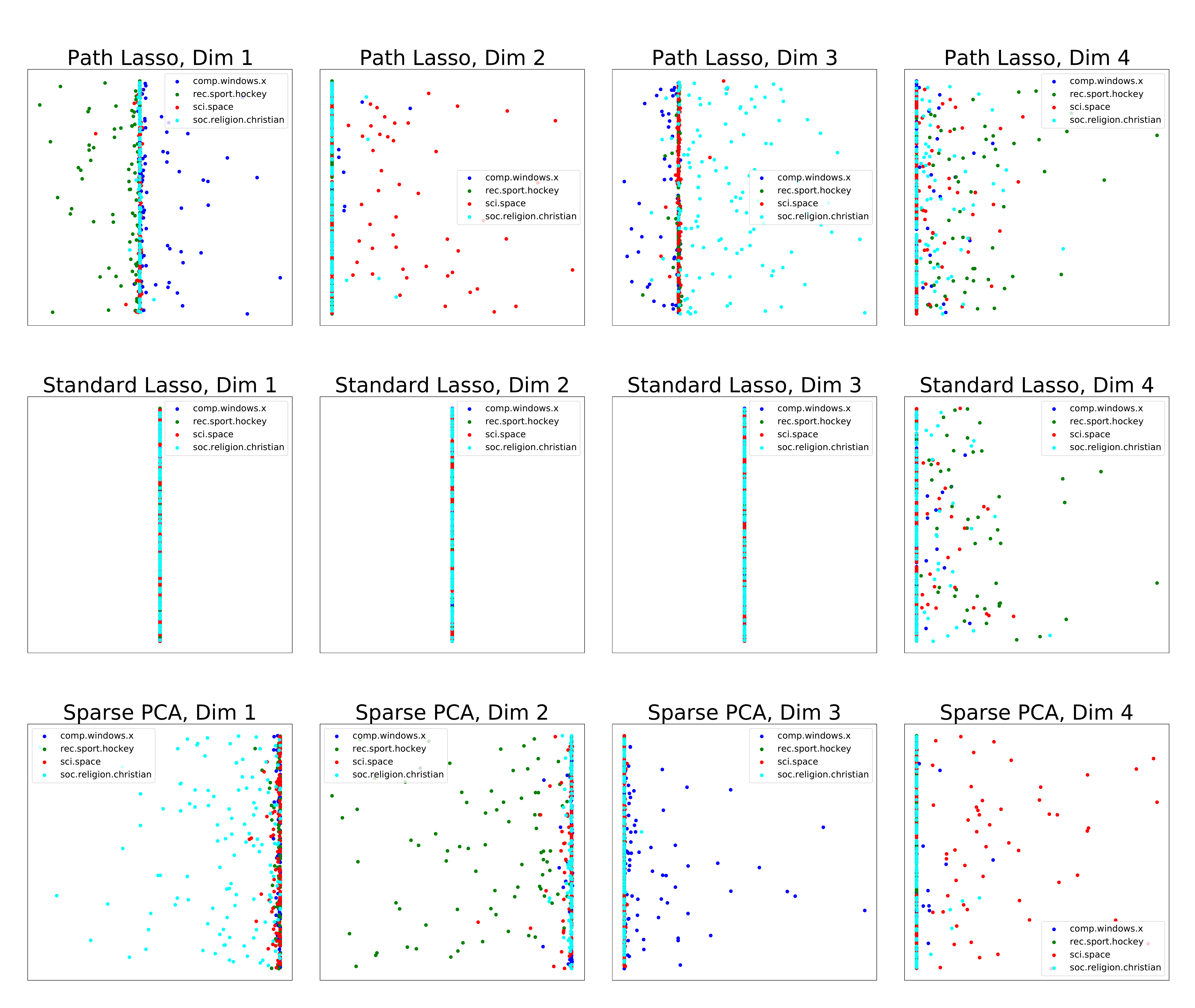}
         \caption{Latent 4D space for the newsgroup data}
         \label{fig:news4}
     \end{subfigure}
     \caption{Latent spaces of the news data, when compressed down to two and to four latent dimensions. In \ref{fig:news4} the y-axis value is just a random number, added to increase readability. Path lasso uses the latent space more efficiently than the other algorithms, being capable of mapping one category to each half-axis.}
     \label{fig:news}
\end{figure*}

\begin{table}
\centering
\begin{tabular}{|c|c|c|c|}
\hline
Algorithm & Dimension & Axis & Words\\
\hline
\hline
\multirow{4}{*}{Path Lasso} & \multirow{2}{*}{1 of 2} & Negative & \makecell{team, hockey, better, probably}\\
\cline{3-4}
 &  & Positive & \makecell{space}\\
\cline{2-4}
 & \multirow{2}{*}{2 of 2} & Negative & \makecell{window, hi}\\
\cline{3-4}
 &  & Positive & \makecell{god, church}\\
\cline{2-4}
\hline
\hline
\multirow{4}{*}{Standard Lasso} & \multirow{2}{*}{1 of 2} & Negative & \makecell{-}\\
\cline{3-4}
 &  & Positive & \makecell{-}\\
\cline{2-4}
 & \multirow{2}{*}{2 of 2} & Negative & \makecell{-}\\
\cline{3-4}
 &  & Positive & \makecell{people, good, team, new, did, hockey,\\ make, going, better, probably, lot}\\
\cline{2-4}
\hline
\hline
\multirow{4}{*}{Sparse PCA} & \multirow{2}{*}{1 of 2} & Negative & \makecell{god, people, jesus, believe,\\ bible, christ, faith, life}\\
\cline{3-4}
 &  & Positive & \makecell{-}\\
\cline{2-4}
 & \multirow{2}{*}{2 of 2} & Negative & \makecell{-}\\
\cline{3-4}
 &  & Positive & \makecell{window, application}\\
\cline{2-4}
\hline
\end{tabular}
\caption{Positive and negative words in the two latent dimension, with sign calculated as sign($W_2\cdot W_1+(W_4\cdot W_3)^\top$).}
\label{tab:news_words2}
\end{table}

\begin{table}
\centering
\begin{tabular}{|c|c|c|c|}
\hline
Algorithm & Dimension & Axis & Words\\
\hline
\hline
\multirow{8}{*}{\makecell{Path\\Lasso}} & \multirow{2}{*}{1 of 4} & Negative & \makecell{game, year, hockey, play,\\ games, players, season, nhl}\\
\cline{3-4}
 &  & Positive & \makecell{window, use, need, server, program, motif,\\ using, windows, application, widget}\\
\cline{2-4}
 & \multirow{2}{*}{2 of 4} & Negative & \makecell{-}\\
\cline{3-4}
 &  & Positive & \makecell{space}\\
\cline{2-4}
 & \multirow{2}{*}{3 of 4} & Negative & \makecell{know, thanks, edu, mail, hi, list}\\
\cline{3-4}
 &  & Positive & \makecell{god, don, think, people, jesus,\\ say, believe, church, christians,\\ christian, bible, christ, faith, life}\\
\cline{2-4}
 & \multirow{2}{*}{4 of 4} & Negative & \makecell{-}\\
\cline{3-4}
 &  & Positive & \makecell{team, new, hockey, going,\\ better, probably, lot}\\
\cline{2-4}
\hline
\hline
\multirow{8}{*}{\makecell{Standard\\Lasso}} & \multirow{2}{*}{1 of 4} & Negative & \makecell{-}\\
\cline{3-4}
 &  & Positive & \makecell{-}\\
\cline{2-4}
 & \multirow{2}{*}{2 of 4} & Negative & \makecell{-}\\
\cline{3-4}
 &  & Positive & \makecell{-}\\
\cline{2-4}
 & \multirow{2}{*}{3 of 4} & Negative & \makecell{-}\\
\cline{3-4}
 &  & Positive & \makecell{-}\\
\cline{2-4}
 & \multirow{2}{*}{4 of 4} & Negative & \makecell{god, like, know, don, think, space, does, christ,\\ time, window, thanks, use, jesus, way, say, sun,\\ believe, need, want, problem, edu, server, hi,\\ right, church, program, using, work, nasa, life,\\ christians, true, bible, help, mail, used, actually}\\
\cline{3-4}
 &  & Positive & \makecell{just, people, good, team, new, did, hockey,\\ make, going, better, probably, lot}\\
\cline{2-4}
\hline
\hline
\multirow{8}{*}{\makecell{Sparse\\PCA}} & \multirow{2}{*}{1 of 4} & Negative & \makecell{god, don, think, people, does, jesus, say,\\ believe, church, things, question, said,\\ christians, true, christian, bible, come,\\ world, point, christ, faith, life}\\
\cline{3-4}
 &  & Positive & \makecell{-}\\
\cline{2-4}
 & \multirow{2}{*}{2 of 4} & Negative & \makecell{game, good, team, year, hockey, play,\\ games, players, season, better, best, nhl}\\
\cline{3-4}
 &  & Positive & \makecell{use}\\
\cline{2-4}
 & \multirow{2}{*}{3 of 4} & Negative & \makecell{-}\\
\cline{3-4}
 &  & Positive & \makecell{window, server, program, motif,\\ using, windows, application, widget}\\
\cline{2-4}
 & \multirow{2}{*}{4 of 4} & Negative & \makecell{-}\\
\cline{3-4}
 &  & Positive & \makecell{space, nasa, earth}\\
\cline{2-4}
\hline
\end{tabular}
\caption{Positive and negative words in the four latent dimension, with sign calculated as sign($W_2\cdot W_1+(W_4\cdot W_3)^\top$).}
\label{tab:news_words4}
\end{table}

\begin{table}
\centering
\begin{tabular}{|l|l|l|l|l|l|}
\hline
\multirow{2}{*}{Dimensions} & \multirow{2}{*}{Algorithm} & \multirow{2}{*}{Connections} & \multirow{2}{*}{$R^2$}
& \multicolumn{2}{c|}{Reconstruction Match}\\
\cline{5-6}
& & & & Observation & Label\\
\hline
2 & Path lasso & 9 & 0.13 & 0.021 & 0.47\\
2 & Standard Lasso & 11 & 0.015 & 0.0021 & 0.25\\
2 & Sparse PCA & 10 & 0.028 & 0.0084 & 0.29\\
\hline
4 & Path lasso & 46 & 0.20 & 0.027 & 0.55\\
4 & Standard Lasso & 47 & 0.015 & 0.0021 & 0.25\\
4 & Sparse PCA & 46 & 0.11 & 0.017 & 0.43\\
\hline
25 & Path Lasso & 102 & 0.55 & 0.34 & 0.66\\
25 & Standard Lasso & 104 & 0.069 & 0.0063 & 0.33\\
25 & Sparse PCA & 104 & 0.40 & 0.14 & 0.56\\
\hline
\end{tabular}
\caption{Number of remaining connections, explained variance and reconstruction match for the three algorithms on the newsgroup data at high sparsity. Path lasso performs best both in terms of explained variance and reconstruction match.}
\label{tab:news_stats_high}
\end{table}

\begin{table}
\centering
\begin{tabular}{|l|l|l|l|l|l|}
\hline
\multirow{2}{*}{Dimensions} & \multirow{2}{*}{Algorithm} & \multirow{2}{*}{Connections} & \multirow{2}{*}{$R^2$}
& \multicolumn{2}{c|}{Reconstruction Match}\\
\cline{5-6}
& & & & Observation & Label\\
\hline
2 & Autoencoder & 200 & 0.24 & 0.048 & 0.56\\
2 & Path Lasso & 194 & 0.27 & 0.055 & 0.54\\
2 & Standard Lasso & 196 & 0.073 & 0.0042 & 0.35\\
2 & Sparse PCA & 195 & 0.051 & 0.0063 & 0.32\\
\hline
4 & Autoencoder & 400 & 0.33 & 0.080 & 0.58\\
4 & Path Lasso & 396 & 0.34 & 0.10 & 0.56\\
4 & Standard Lasso & 398 & 0.19 & 0.019 & 0.45\\
4 & Sparse PCA & 397 & 0.14 & 0.015 & 0.43\\
\hline
25 & Autoencoder & 2500 & 0.60 & 0.47 & 0.71\\
25 & Path Lasso & 2475 & 0.61 & 0.49 & 0.72\\
25 & Standard Lasso & 2476 & 0.44 & 0.16 & 0.61\\
25 & Sparse PCA & 2486 & 0.44 & 0.15 & 0.57\\
\hline
\end{tabular}
\caption{Number of remaining connections, explained variance and reconstruction match for the four algorithms on the newsgroup data at low sparsity. Path lasso performs better than standard lasso and sparse PCA both in terms of explained variance and reconstruction match. Compared to the standard autoencoder path lasso performs slightly better in terms of explained variance and observation reconstruction match.}
\label{tab:news_stats_low}
\end{table}

\subsection{Images - AT\&T Face Database}
We also tested the algorithm on the AT\&T face database \citep{samaria1994parameterisation}, which contains 400 grayscale images of faces. The images were compressed to a size of $60\times 50$ pixels, after which the 3000 dimensional images where reduced to 5 dimensions. Both the encoder and the decoder had 1000 units wide hidden layers, and to assure that a pixel that was disconnected from all the latent dimensions got a value of zero, no bias parameters were used.

Again, each algorithm was penalized to obtain (approximately) the same number of connections, at two different sparsity levels. One low, with almost all of the $5\cdot 3000=15000$ connections kept to be able to compare to a dense, $l_2$-regularized autoencoder; and one high, with only a fourth of the connections kept, to compare the algorithms at more extreme sparsity levels.

The results are summarized in Table \ref{tab:att}. Since the images are unlabeled the label reconstruction match was replaced by a 10 nearest neighbor reconstruction match, where an observation is considered correctly reconstructed if its own original observation is among the ten closest original observations.

Path lasso and standard lasso do much better than sparse PCA in terms of reconstruction, and at low sparsity they are on par with the dense autoencoder. For observation reconstruction match, path lasso outperforms the other two algorithms, and at low sparsity also the standard autoencoder. For the 10 nearest neighbors reconstruction match, all non-linear algorithms do very well at low sparsity, while at high sparsity path lasso does better than the other algorithms.

The five eigenfaces, calculated as Dec$(e_i)$ where Dec is the decoder and $e_i$ is the i-th standard basis vector in $\R^5$, are shown in the left column of Figure \ref{fig:att}. For PCA, the decoder function corresponds to multiplication with the loadings matrix, whose i-th column is the i-th eigenface. A pixel with value zero is colored red.
Sparse PCA distributes its non-zero pixels more evenly among the latent dimensions, while again standard lasso does not use all 5 latent dimensions.
The right column shows the reconstruction of some of the images in the test set, using the four different algorithms. The reconstructed images using path lasso look more diversified than for the other sparse algorithms, something that is in line with the superior reconstruction match of path lasso.

\begin{table}
\centering
\begin{tabular}{|l|l|l|l|l|}
\hline
\multirow{2}{*}{Algorithm} & \multirow{2}{*}{Connections} & \multirow{2}{*}{$R^2$}
& \multicolumn{2}{c|}{Reconstruction Match}\\
\cline{4-5}
& & & Observation & 10 Nearest Neighbors\\
\hline
Autoencoder & 15000 & 0.73 & 0.66 & 0.97\\
\hline
Path Lasso & 14923 & 0.72 & 0.70 & 0.96\\
Standard Lasso & 14936 & 0.70 & 0.55 & 0.96\\
Sparse PCA & 14923 & -0.24 & 0.16 & 0.60\\
\hline
Path Lasso & 3770 & 0.57 & 0.21 & 0.75\\
Standard Lasso & 3771 & 0.44 & 0.025 & 0.26\\
Sparse PCA & 3770 & -1.2 & 0.10 & 0.46\\
\hline
\end{tabular}
\caption{Number of remaining connections, explained variance and reconstruction match for the four algorithms on the image data. Path lasso outperforms the competing algorithms in terms of observation reconstruction match, and at low sparsity even beats the fully connected autoencoder.}
\label{tab:att}
\end{table}

\begin{figure*}
     \raggedleft
     \begin{subfigure}[b]{0.33\textwidth}
         \centering
         \includegraphics[width=\textwidth]{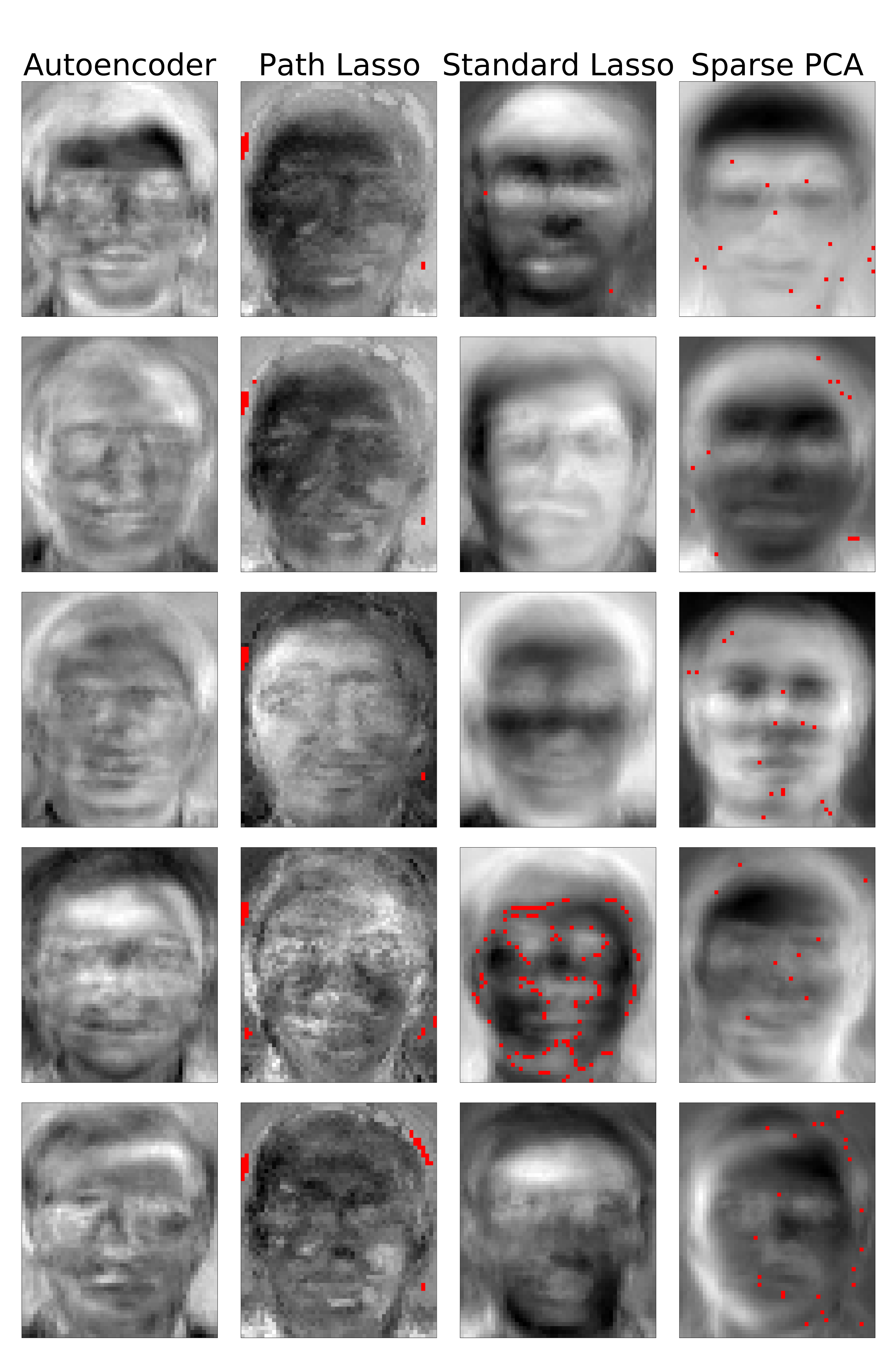}
         \caption{Eigenfaces at low sparsity}
         \label{fig:att_eig2}
     \end{subfigure}
     \hfill
     \raggedright
     \begin{subfigure}[b]{0.40\textwidth}
         \centering
         \includegraphics[width=\textwidth]{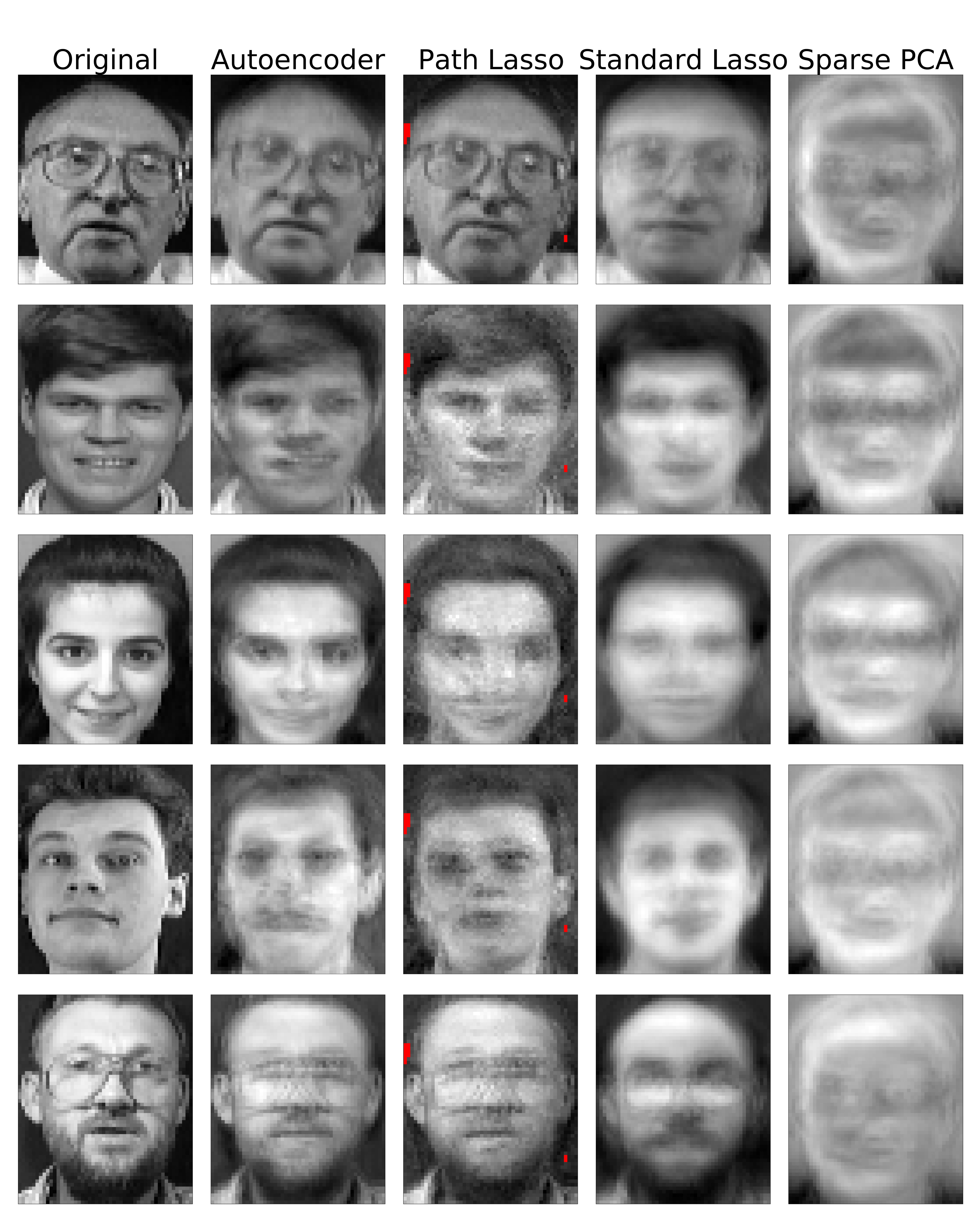}
         \caption{Reconstructions at low sparsity}
         \label{fig:att_rec2}
     \end{subfigure}
     \raggedleft
     \begin{subfigure}[b]{0.30\textwidth}
         \centering
         \includegraphics[width=\textwidth]{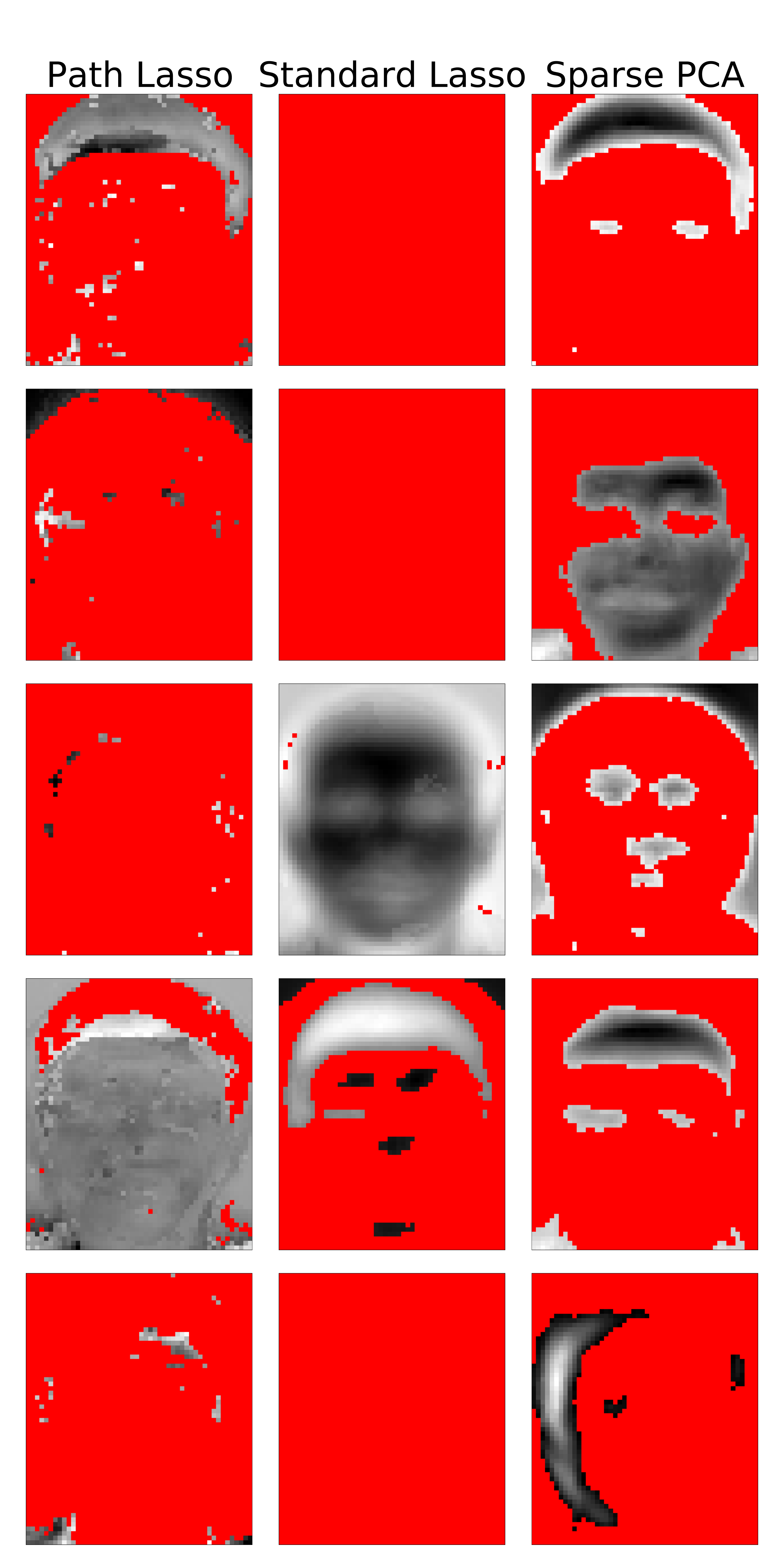}
         \caption{Eigenfaces at high sparsity}
         \label{fig:att_eig1}
     \end{subfigure}
     \hfill
     \raggedright
     \begin{subfigure}[b]{0.40\textwidth}
         \centering
         \includegraphics[width=\textwidth]{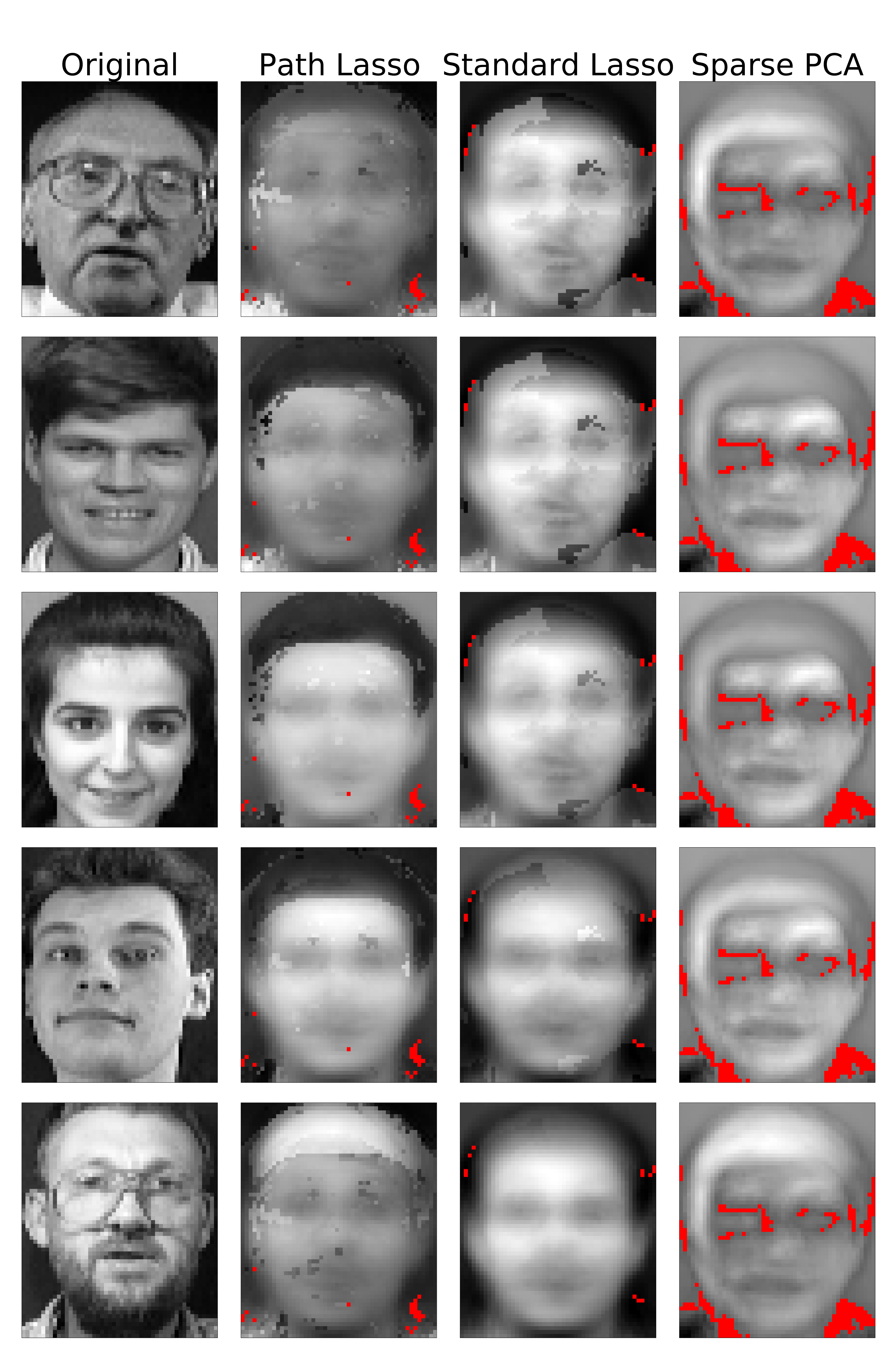}
         \caption{Reconstructions at high sparsity}
         \label{fig:att_rec1}
     \end{subfigure}
     \caption{Eigenfaces and examples of reconstructions of test faces at small and large penalization. Zeros are indicated with red. The results of the dense autoencoder are presented together with the low sparsity results. The reconstructed images using path lasso look more diverse than those of the other sparse algorithms. At high sparsity standard lasso does not use all 5 latent dimensions.}
     \label{fig:att}
\end{figure*}

\section{Conclusions}
\label{sec:conculsions}
We proposed path lasso, a penalty that creates structured sparsity in feedforward neural networks, by using a group lasso penalty to remove all connections between two nodes in two non-adjacent layers. We applied path lasso to an autoencoder to obtain sparse, non-linear dimensionality reduction, and showed that this non-linearity makes path lasso much more flexible than sparse PCA, leading to a lower reconstruction error and a higher reconstruction match. Thanks to its higher flexibility path lasso is also able to use the latent space more efficiently, something that proved essential when the latent dimensions were few. Compared to an autoencoder with individually lasso penalized links, path lasso performed better in terms of reconstruction, reconstruction match and interpretation of the latent space. In addition, at low sparsity levels path lasso resulted in a better reconstruction match than the standard autoencoder.

Using path lasso for non-linear, sparse dimensionality reduction, flexibility and interpretability can be combined in a new way, enabling compression to fewer dimensions, with preserved interpretability. In this paper, we used the path lasso penalty in an autoencoder. However, path lasso can be used in many other types of feedforward neural networks with applications including non-linear, sparse multivariate regression, and non-linear, sparse network models. Such investigations are left for future work.

Code is available at \url{https://github.com/allerbo/path_lasso}.

\section*{Acknowledgments}
We would like to thank the anonymous reviewers, whose suggestions helped improve and clarify the manuscript.

This research was supported by funding from the Swedish Research Council (VR), the Swedish Foundation for Strategic Research, the Wallenberg AI, Autonomous Systems and Software Program (WASP), and the Chalmers AI Research Center (CHAIR).

\clearpage
\appendix

\section{Accelerating the Non-Negative Matrix Factorization}
\label{sec:speed}
To increase the speed of the matrix factorization step in Equation \eqref{eq:eq_syst_mat} the following three approaches were used: Substitution, parallelization and Boolean matrix factorization.

\subsection{Substitution}
\label{sec:substitution}
By first applying an off-the-shelf optimization method to $f(\bm{\theta})=g(\bm{\theta})+\lambda\Wgl$, where $g(\bm{\theta})$ is the reconstruction error, we obtain a solution where some of the paths have very small values, although non-zero. We then, as in adaptive lasso, use this solution to initialize a second optimization stage, using proximal gradient descent, with individual penalties for each connection, where the penalties depend on the magnitude of the connection after the first stage:
\begin{equation*}
\lambda_{i_L,i_0} := \frac{\lambda}{((\Wglhat)_{i_Li_0})^\gamma}
\end{equation*}
where $\Wglhat$ is the value of $\Wgl$ after the optimization with the off-the-shelf optimizer and $\gamma >0$. Just as in Section \ref{sec:apply_pen}, $\gamma=2$ was used. Choosing a relatively small value of $\lambda$, connections that are not close to zero after the first stage will be left more or less unpenalized, while connections close to zero will be penalized hard. Another advantage of this two stage procedure is that we are able to further increase the speed in the first stage by leveraging on momentum based optimization methods, which are not available for proximal gradient descent.

\subsection{Parallelization}
In general, the larger the weight matrices, the more time consuming is the factorization of the penalized path matrix into penalized weight matrices. By splitting the weight matrices into blocks, the factorization can instead be done in parallel for smaller sub-matrices. In the simplest example, three layers are split into two parts each. This corresponds to two weight matrices, whose rows and columns are both split into two blocks each, and is shown in Equation \eqref{eq:block_fact} below. Let $\bm{A} \in \Rr^{d_3\times d_2}$ and $\bm{B} \in \Rr^{d_2\times d_1}$ denote the absolute valued weight matrices and let $\bm{C}+\bm{D} \in \Rr^{d_3\times d_1}$ denote the penalized path matrix, where the sum of the paths in $(\bm{C}+\bm{D})_{ij}$ is split into $\bm{C}_{ij}$, containing the sum of some of the paths, and $\bm{D}_{ij}$, containing the sum of the remaining paths. Which paths belong to $\bm{C}$ and which belong to $\bm{D}$ depends on how $\bm{A}$ and $\bm{B}$ are split. Then, the factorization of $\bm{C}+\bm{D}$ into $\bm{A}\cdot \bm{B}$ can be divided into $2^3= 8$ smaller factorizations which can be solved in parallel, i.e.

\begin{equation}
\label{eq:block_fact}
\left[
\begin{array}{c|c}
\bm{C_{11}}+\bm{D_{11}} & \bm{C_{12}}+\bm{D_{12}}\\
\hline
\bm{C_{21}}+\bm{D_{21}} & \bm{C_{22}}+\bm{D_{22}}\\
\end{array}
\right]
=
\left[
\begin{array}{c|c}
\bm{A_{11}} & \bm{A_{12}}\\
\hline
\bm{A_{21}} & \bm{A_{22}}\\
\end{array}
\right]
\cdot
\left[
\begin{array}{c|c}
\bm{B_{11}} & \bm{B_{12}}\\
\hline
\bm{B_{21}} & \bm{B_{22}}\\
\end{array}
\right]\\
\end{equation}
can be split into
\begin{equation*}
\begin{aligned}
&\left[
\begin{array}{c|c}
\bm{C_{11}} & \bm{C_{12}}\\
\hline
\bm{C_{21}} & \bm{C_{22}}\\
\end{array}
\right]
=
\left[
\begin{array}{c|c}
\bm{A_{11}} & \bm{0}\\
\hline
\bm{A_{21}} & \bm{0}\\
\end{array}
\right]
\cdot
\left[
\begin{array}{c|c}
\bm{B_{11}} & \bm{B_{12}}\\
\hline
\bm{0} & \bm{0} \\
\end{array}
\right]
=
\left[
\begin{array}{c|c}
\bm{A_{11}}\cdot \bm{B_{11}} & \bm{A_{11}}\cdot \bm{B_{12}}\\
\hline
\bm{A_{21}}\cdot \bm{B_{11}} & \bm{A_{21}}\cdot \bm{B_{12}}\\
\end{array}
\right] \\
&\left[
\begin{array}{c|c}
\bm{D_{11}} & \bm{D_{12}}\\
\hline
\bm{D_{21}} & \bm{D_{22}}\\
\end{array}
\right]
=
\left[
\begin{array}{c|c}
\bm{0} & \bm{A}_{12}\\
\hline
\bm{0} & \bm{A}_{22}\\
\end{array}
\right]
\cdot
\left[
\begin{array}{c|c}
\bm{0} & \bm{0} \\
\hline
\bm{B_{21}} & \bm{B_{22}}\\
\end{array}
\right]
=
\left[
\begin{array}{c|c}
\bm{A_{12}}\cdot \bm{B_{21}} & \bm{A_{12}}\cdot \bm{B_{22}}\\
\hline
\bm{A_{22}}\cdot \bm{B_{21}} & \bm{A_{22}}\cdot \bm{B_{22}}\\
\end{array}
\right] \\
\end{aligned}
\end{equation*}
or equivalently
\begin{equation*}
\begin{aligned}
&\bm{C_{11}}=\bm{A_{11}}\cdot \bm{B_{11}},\ \bm{C_{12}}=\bm{A_{11}}\cdot \bm{B_{12}},\ \bm{C_{21}}=\bm{A_{21}}\cdot \bm{B_{11}},\ \bm{C_{22}}=\bm{A_{21}}\cdot \bm{B_{12}}\\
&\bm{D_{11}}=\bm{A_{12}}\cdot \bm{B_{21}},\ \bm{D_{12}}=\bm{A_{12}}\cdot \bm{B_{22}},\ \bm{D_{21}}=\bm{A_{22}}\cdot \bm{B_{21}},\ \bm{D_{22}}=\bm{A_{22}}\cdot \bm{B_{22}}.
\end{aligned}
\end{equation*}

Here, $\bm{A_{ij}} \in \Rr^{d_{3i}\times d_{2j}}$, $\bm{B_{ij}} \in \Rr^{d_{2i}\times d_{1j}}$ and $\bm{C_{ij}}, \bm{D_{ij}} \in \Rr^{d_{3i}\times d_{1j}}$, where $i,j \in \{1,2\}$ and $d_{ki}+d_{kj}=d_k$ for $k=1,2,3$, are all block matrices. This results in two, potentially different, solutions for each $\bm{A_{ij}}$ and $\bm{B_{ij}}$, that need to be aggregated. To be conservative, and set a link to zero only when it is zero in both solutions, we use the maximum value element-wise for aggregation:
\begin{equation*}
(\bm{A_{ij}})_{kl} = \max_{n\in \{1,2\}} (\bm{A^n_{ij}})_{kl}
\end{equation*}
where $\bm{A^n_{ij}}$ is the $n$-th solution of $\bm{A_{ij}}$. Remember that all elements in all matrices are non-negative.

This generalizes trivially to more layers and splits and the number of equations is always the product of the splits over all the layers.

\subsection{Boolean Matrix Factorization}
Another way to speed up the computations is early stopping of the matrix factorization. This may however lead to some of the elements on the right hand side in Equation \eqref{eq:eq_syst_mat} being very close, but not equal, to zero, introducing the need to threshold. We determine the optimal threshold value using Boolean matrix multiplication, which differs from ordinary matrix multiplication in the following way:
\begin{itemize}
\item All elements are in $\{0,1\}$.
\item Instead of ordinary sum, Boolean sum (or Boolean \textbf{or}), $\vee$, is used: $0 \vee 0 = 0$, $0 \vee 1 = 1$, $1 \vee 0 = 1$, $1 \vee 1 = 1$.
\end{itemize}
We define a Boolean matrix for each matrix in Equation \eqref{eq:eq_syst_mat} as:
\begin{equation*}
\begin{aligned}
(\bm{B})_{ij} &:=\mathbb{I}\left[\left(\prod_{l=L,\dots,1} |\bm{W_l}|\odot\left(1-\frac{\alpha\lambda}{\Wgl}\right)^+\right)_{ij} > 0\right]\\
(\bm{B_l})_{ij}(\tau)&:=\mathbb{I}\left[(\bm{\tilde{W}_l})_{ij} > \tau\right],\ \tau\geq 0
\end{aligned}
\end{equation*}
where $\mathbb{I[\cdot]}$ is the (element-wise) indicator function which equals 1 when its argument is true and 0 otherwise. Thus, in the matrix on the left hand side in Equation \eqref{eq:eq_syst_mat}, zero elements become zeros and all other elements become ones, while for the matrices on the right hand side, we use a threshold to decide which elements should be set to zero. In the first case, an element of one means that there is an non-zero connection, while in the second case it means that there is an non-zero link. If $\bm{B} = \bigvee_{l=L,\dots,1} \bm{B_l}$, where $\bigvee$ denotes Boolean matrix multiplication, then the matrix factorization was successful in the sense that all connections in $\Wgl$ are represented by appropriate links in the weight matrices. Each differing element means that either two nodes are disconnected in $\Wgl$ but not in $\{\bm{\tilde{W}_l}\}_{l=1}^L$ or vice versa. To minimize this quantity, we minimize $\sum |\bm{B} - \bigvee_{l=L,\dots,1} \bm{B_l}(\tau)|$, with the absolute value taken element-wise, with respect to $\tau$, and then threshold $\{\bm{\tilde{W}_l}\}_{l=1}^L$ using the resulting $\tau$:
\begin{equation*}
\begin{aligned}
&\tau = \underset{\tau\geq 0}{\arg\!\min}\sum_{\textrm{all elements}}\left|\mathbb{I}\left[\prod_{l=L,\dots,1} |\bm{W_l}|\odot\left(1-\frac{\alpha\lambda}{\Wgl}\right)^+>0\right]-\bigvee_{l=L,\dots,1}\mathbb{I}[\bm{\tilde{W}_l}>\tau]\right|\\
&\bm{\tilde{W}_l} \leftarrow \bm{\tilde{W}_l} \odot \mathbb{I}[\bm{\tilde{W}_l}>\tau]
\end{aligned}
\end{equation*}
Since the objective function is piecewise constant and thus hard to optimize, we simply evaluate it for 20 logarithmically spaced values of $\tau \in [10^{-10},1]$ and pick the best one.

\section{Proofs}
\label{sec:proofs}
\begin{proof}[Proof of Proposition \ref{thm:GL_eq_paths}]
Using the definition of matrix multiplication, with $k$ running over all paths connecting $x_{i_0}$ and $y_{i_L}$, we obtain
\begin{equation*}
\begin{aligned}
\sqrt{\sum_{k\in g_{i_Li_0}} (p^k_{i_Li_0})^2}&=\sqrt{\sum_{i_{L-1}=1}^{d_{L-1}}\sum_{i_{L-2}=1}^{d_{L-2}}\dots\sum_{i_1=1}^{d_1} (|w^L_{i_Li_{L-1}}|\cdot |w^{L-1}_{i_{L-1}i_{L-2}}|\cdot \dots |w^1_{i_1i_0}|)^2}\\
&=\sqrt{\sum_{i_{L-1}=1}^{d_{L-1}}\sum_{i_{L-2}=1}^{d_{L-2}}\dots\sum_{i_1=1}^{d_1} (w^L_{i_Li_{L-1}})^2\cdot (w^{L-1}_{i_{L-1}i_{L-2}})^2\cdot \dots (w^1_{i_1i_0})^2}\\
&=\left(\sqrt{(\bm{W_L})^2\cdot(\bm{W_{L-1}})^2\cdot \dots (\bm{W_1})^2}\right)_{i_Li_0}=(\Wgl)_{i_Li_0}.
\end{aligned}
\end{equation*}
\end{proof}

To prove Proposition \ref{thm:GL_der} we begin with the following lemma:
\begin{lemma}\label{lemma:mat_sq}
Let $\{\bm{A_k}\}_{k=1}^n$ be a set of arbitrary matrices and $\{\bm{D_k}\}_{k=1}^n$ a set of diagonal matrices, where all elements in $\bm{A_k}$ and $\bm{D_k}$ are bounded. Let $f\geq 0$ an element-wise function where $f(x)=0$ if and only if $x=0$. Let the dimensions of the matrices be such that the matrix multiplications below make sense. Then
$$(f(\bm{A_1})\cdot f(\bm{A_2}) \cdot \dots f(\bm{A_n}))_{ij} = 0 \implies (\bm{D_1}\cdot \bm{A_1}\cdot \bm{D_2}\cdot \bm{A_2}\cdot \bm{D_2}\cdot \dots\, \bm{D_n} \cdot \bm{A_n})_{ij}.$$
\end{lemma}
\begin{proof}
Denote $a^k_{ij}:=(A_k)_{ij}$ and $d^k_{ij}:=(D_k)_{ij}$. Then for the left hand side
\begin{equation*}
(f(\bm{A_1})\cdot f(\bm{A_2})\cdot \dots f(\bm{A_n}))_{ij} = \sum_{k_1}\sum_{k_2}\dots \sum_{k_{n-1}}f(a^1_{ik_1})\cdot f(a^2_{k_1k_2})\cdot \dots f(a^n_{k_{n-1}j})
\end{equation*}
and for the right hand side
\begin{equation*}
(\bm{D_1} \cdot \bm{A_1}\cdot \bm{D_2} \cdot \bm{A_2} \cdot \dots\, \bm{D_n} \cdot \bm{A_n})_{ij}=\sum_{k_1}\sum_{k_2}\dots \sum_{k_{n-1}} a^1_{ik_1}\cdot a^2_{k_1k_2}\cdot \dots a^n_{k_{n-1}j}\cdot d^1_{ii}\cdot d^2_{k_1k_1}\cdot \dots d^n_{k_{n-1}k_{n-1}}.
\end{equation*}
By the definition of $f$, $f(a^1_{ik_1})\cdot f(a^2_{k_1k_2})\cdot \dots f(a^n_{k_{n-1}j}) \geq 0$ with equality only if at least one of $a^1_{ik_1},\ a^2_{k_1k_2},\dots a^n_{k_{n-1}j}$ is zero, which implies that $a^1_{ik_1}\cdot a^2_{k_1k_2}\cdot \dots a^n_{k_{n-1}j} =0$.

The sum
$$\sum_{k_1}\sum_{k_2}\dots \sum_{k_{n-1}}f(a^1_{ik_1})\cdot f(a^2_{k_1k_2})\cdot \dots f(a^n_{k_{n-1}j})$$
is zero if and only if all its terms are zero, which means that
$$a^1_{ik_1}\cdot a^2_{k_1k_2}\cdot \dots a^n_{k_{n-1}j} = 0$$
for all combinations of indices summed over. Multiplying with some bounded constant does not change that fact, so

$$f(a^1_{ik_1})\cdot f(a^2_{k_1k_2})\cdot \dots f(a^n_{k_{n-1}j})=0 \implies a^1_{ik_1}\cdot a^2_{k_1k_2}\cdot \dots a^n_{k_{n-1}j}\cdot d^1_{ii}\cdot d^2_{k_1k_1}\cdot \dots d^n_{k_{n-1}k_{n-1}}= 0.$$
Finally summing only zeros we get
\begin{equation*}
\begin{aligned}
0&=\sum_{k_1}\sum_{k_2}\dots \sum_{k_{n-1}} a^1_{ik_1}\cdot a^2_{k_1k_2}\cdot \dots a^n_{k_{n-1}j}\cdot d^1_{ii}\cdot d^2_{k_1k_1}\cdot \dots d^n_{k_{n-1}k_{n-1}}\\
&=(\bm{D_1} \cdot \bm{A_1}\cdot \bm{D_2} \cdot \bm{A_2} \cdot \dots\, \bm{D_n} \cdot \bm{A_n})_{ij}
\end{aligned}
\end{equation*}
\end{proof}

\begin{proof}[Proof of Proposition \ref{thm:GL_der}]
With $\{\bm{o_l}\}_{l=0}^L$ (where $\bm{o_0}=\bm{x}$ and $\bm{o_L}=\bm{y}$) denoting the outputs and $\{\bm{i_l}\}_{l=1}^L$ the inputs of the activation functions $\{\Phi_l\}_{l=1}^L$, we can express Equation \eqref{eq:nn} recursively for $1\leq l \leq L$ as
\begin{equation*}
\begin{aligned}
\bm{i_l}&=\bm{W_l}\bm{o_{l-1}}+\bm{b_l}\\
\bm{o_l} &= \Phi_k(\bm{i_l})\\
\end{aligned}
\end{equation*}
Applying the chain rule we obtain
\begin{equation}
\begin{aligned}
\frac{\partial \bm{y}}{\partial \bm{x}} = \frac{\partial \bm{o_L}}{\partial \bm{o_0}} = \frac{\partial \bm{o_{L}}}{\partial \bm{i_{L}}} \cdot \frac{\partial \bm{i_{L}}}{\partial \bm{o_{L-1}}} \cdot \frac{\partial \bm{o_{L-1}}}{\partial \bm{i_{L-1}}} \cdot \frac{\partial \bm{i_{L-1}}}{\partial \bm{o_{L-2}}} \cdot \dots \frac{\partial \bm{o_{1}}}{\partial \bm{i_{1}}} \cdot \frac{\partial \bm{i_{1}}}{\partial \bm{o_{0}}},
\end{aligned}
\label{eq:deriv}
\end{equation}
where $\frac{\partial \bm{o_{l}}}{\partial \bm{i_{l}}}$ is a diagonal matrix, since $\Phi_l$ is an element-wise operator, and $\frac{\partial \bm{i_{l}}}{\partial \bm{o_{l-1}}}=\bm{W_l}$. We can now apply Lemma \ref{lemma:mat_sq} with $f(x)=x^2$ to Equation \eqref{eq:deriv} to obtain 
\begin{equation*}
\begin{aligned}
&(\Wgl)_{i_Li_0}=0\iff((\bm{W_{L}})^2\cdot (\bm{W_{L-1}})^2\cdot \dots (\bm{W_1})^2)_{i_Li_0} = 0 \\
\implies &0= \left(\frac{\partial \bm{o_{L}}}{\partial \bm{i_{L}}} \cdot \bm{W_L} \cdot \frac{\partial \bm{o_{L-1}}}{\partial \bm{i_{L-1}}} \cdot \bm{W_{L-1}} \cdot \dots \frac{\partial \bm{o_{1}}}{\partial \bm{i_{1}}} \cdot \bm{W_1}\right)_{i_Li_0}=\left(\frac{\partial \bm{y}}{\partial \bm{x}}\right)_{i_Li_0}
\end{aligned}
\end{equation*}
where the equivalence comes from the definition of $\Wgl$ and the implication from Lemma \ref{lemma:mat_sq}.
\end{proof}

\begin{proof}[Proof of Proposition \ref{thm:paths_to_links}]
Assuming the penalized path can be written as a product of penalized links, we get the following equation for the k-th path between nodes $i_0$ and $i_L$, in total $\prod_{l=0}^{L}d_k$ equations:
\begin{equation*}
\begin{aligned}
&p^k_{i_Li_0}\cdot\left(1-\frac{\alpha\lambda}{(\Wgl)_{i_Li_0}}\right)^+= |w^L_{i_Li_{L-1}^k}|\cdot |w^{L-1}_{i_{L-1}^ki_{L-2}^k}|\cdot \dots |w^1_{i_1^ki_0}|\cdot\left(1-\frac{\alpha\lambda}{(\Wgl)_{i_Li_0}}\right)^+\\
&= \softthresha{w^L_{i_Li_{L-1}^k}}{\alpha\lambda^L_{i_Li_{L-1}^k}}\cdot \dots \softthresha{w^1_{i_1^ki_0}}{\alpha\lambda^1_{i_1^ki_0}}\\
&=: \softthreshb{w^L_{i_Li_{L-1}^k}}{\tilde{w}^L_{i_Li_{L-1}^k}}\cdot \dots \softthreshb{w^1_{i_1^ki_0}}{\tilde{w}^1_{i_1^ki_0}}.\\
\end{aligned}
\end{equation*}
The second equality is our assumption, requiring that the proximal operator from Equation \eqref{eq:gl_prox} can be written as a product of proximal operators from Equation \eqref{eq:lasso_prox}. $\lambda^l_{i_l^ki_{l-1}^k}\in [0,|w^l_{i_l^ki_{l-1}^k}|/\alpha]$ is a, possibly unique, penalty for the weight $w^l_{i_l^ki_{l-1}^k}$ and $\tilde{w}^l_{i_l^ki_{l-1}^k}:=(|w^l_{i_l^ki_{l-1}^k}|-\alpha\lambda^l_{i_l^ki_{l-1}^k})^+$. The superscript $k$ on the indices marks that different paths go through different nodes in the inner layers.

Taking absolute values of all equations and summing over the equations where the paths belong to the same connection, i.e.\ they share values for $i_0$ and $i_L$ and thus have the same path penalty, we obtain, for a given combination of $i_0$ and $i_L$
\begin{equation*}
\left(\sum_{i_{L-1}=1}^{d_{L-1}} \dots \sum_{i_1=1}^{d_1}|w^L_{i_Li_{L-1}}|\cdot \dots |w^1_{i_1i_0}|\right)\cdot\left(1-\frac{\alpha\lambda}{(\Wgl)_{i_Li_0}}\right)^+= \sum_{i_{L-1}=1}^{d_{L-1}} \dots \sum_{i_1=1}^{d_1}\tilde{w}^L_{i_Li_{L-1}}\cdot \dots \tilde{w}^1_{i_1i_0}
\end{equation*}
which, using the definition of matrix multiplication can be written as
\begin{equation*}
\left(\prod_{l=L,\dots,1}|\bm{W_l}|\odot\left(1-\frac{\alpha\lambda}{(\Wgl)_{i_Li_0}}\right)^+\right)_{i_Li_0} =  \left(\prod_{l=L,\dots,1}\bm{\tilde{W}_l}\right)_{i_Li_0}.
\end{equation*}
Thus, solving Equation \eqref{eq:eq_syst_mat} results in (absolute) links values being shifted towards zero in such a way that when multiplied into paths, the paths have the correct penalization. The only thing left to do is to restore the signs of the links.
\end{proof}

\section{Modified Non-Negative Matrix Factorization}
\label{sec:nmf}
We solve the non-negative matrix factorization problem in Equation \eqref{eq:eq_syst_mat} with coordinate descent, using a modified version of the solver in python's scikit-learn module \citep{scikit-learn}, which in turn is based on work by \citet{cichocki2009fast} and \citet{hsieh2011fast}. The aim is to minimize $||\bm{V}-\bm{W}\cdot\prod_{i=1}^I \bm{M_i}\cdot \bm{H}||^2_F$ for $I\in\mathbb{N}_0$, keeping each entry in $\bm{W}$, $\bm{M_i}$ and $\bm{H}$ between zero and its seed.

Coordinate descent updates each matrix separately, keeping the others fixed. Since we can write $\bm{W}\cdot \prod_{i=1}^I \bm{M_i} \cdot \bm{H} = \bm{\tilde{W}}\cdot \bm{\tilde{M}} \cdot \bm{\tilde{H}}$, where the three matrices on the right hand side are products of the matrices of the left hand side, when updating a given matrix, we can always treat the problem as a product of only three matrices. Thus we need update rules for $\bm{\tilde{W}}$ and $\bm{\tilde{M}}$, since we can use the same algorithm for $\bm{\tilde{W}}$ and $\bm{\tilde{H}}$ by taking transpose. \citet{hsieh2011fast} describe the case for two matrices, i.e.\ $I=0$. Our contribution is the update rule for $\bm{\tilde{M}}$ (and trivially adding the upper constraint on the solution).

\subsection{Update Rule for \texorpdfstring{$\bm{\tilde{W}}$}{W}}
For $\bm{\tilde{W}}\in (\R^+)^{d_i\times d_r}$ we want to solve, adding the possibility to add element-wise $l_1$- and $l_2$-penalties:

\begin{equation*}
\begin{aligned}
&\min_{\bm{\tilde{W}}:\ 0\leq (\bm{\tilde{W}})_{ir} \leq (\bm{\tilde{W}_0})_{ir}}\frac{1}{2} ||\bm{V}-\bm{\tilde{W}}\bm{\tilde{H}}||_F^2 + \sum_{i,r}\left(\lambda_1(\bm{\tilde{W}})_{ir} + \frac{\lambda_2}{2}(\bm{\tilde{W}})_{ir}^2\right)\\
\end{aligned}
\end{equation*}
where $\bm{\tilde{W}_0}$ is the seed.
We update each element $(\bm{\tilde{W}})_{ir}$ by adding $s\bm{E_{ir}}$, for an optimal $s$, where $\bm{E_{ir}}$ is a matrix with all zeros, except element $(i,r)$, which is 1. Defining
\begin{equation*}
g^{\bm{\tilde{W}}}_{ir}(s):=\frac{1}{2}\sum_{j}((\bm{V})_{ij}-((\bm{\tilde{W}}+s\bm{E_{ir}})\bm{\tilde{H}})_{ij})^2+\lambda_1((\bm{\tilde{W}})_{ir}+s) + \frac{\lambda_2}{2}((\bm{\tilde{W}})_{ir}+s)^2
\end{equation*}
we get
\begin{equation*}
\begin{aligned}
g^{\bm{\tilde{W}}}_{ir}(s)&=\frac{1}{2}\sum_{j}((\bm{V})_{ij}-((\bm{\tilde{W}}+s\bm{E_{ir}})\bm{\tilde{H}})_{ij})^2+\lambda_1((\bm{\tilde{W}})_{ir}+s) + \frac{\lambda_2}{2}((\bm{\tilde{W}})_{ir}+s)^2\\
&=\frac{1}{2}\sum_{j}((\bm{V})_{ij}-(\bm{\tilde{W}}\bm{\tilde{H}})_{ij}-s(\bm{\tilde{H}})_{rj})^2+\lambda_1((\bm{\tilde{W}})_{ir}+s) + \frac{\lambda_2}{2}((\bm{\tilde{W}})_{ir}+s)^2\\
(g^{\bm{\tilde{W}}}_{ir})^\prime(s)&=\sum_{j}(-(\bm{V})_{ij}(\bm{\tilde{H}})_{rj}+(\bm{\tilde{W}}\bm{\tilde{H}})_{ij}(\bm{\tilde{H}})_{rj}+2s(\bm{\tilde{H}})_{rj}^2)+\lambda_1 + \lambda_2((\bm{\tilde{W}})_{ir}+s)\\
(g^{\bm{\tilde{W}}}_{ir})^{\prime\prime}(s)&=\sum_{j}(2(\bm{\tilde{H}})_{rj}^2)+\lambda_2\\
(g^{\bm{\tilde{W}}}_{ir})^\prime(0)&=\sum_{i,j}(-(\bm{V})_{ij}(\bm{\tilde{H}}^{\top})_{jr}+(\bm{\tilde{W}}\bm{\tilde{H}})_{ij}(\bm{\tilde{H}}^\top)_{jr})+\lambda_1+\lambda_2(\bm{\tilde{W}})_{ir}\\
&=(-\bm{V}\bm{\tilde{H}}^\top+\bm{\tilde{W}}\bm{\tilde{H}}\bm{\tilde{H}}^\top)_{ir}+\lambda_1+\lambda_2(\bm{\tilde{W}})_{ir}\\
(g^{\bm{\tilde{W}}}_{ir})^{\prime\prime}(0)&=\sum_j(\bm{\tilde{H}})_{rj}(\bm{\tilde{H}}^\top)_{jr}+\lambda_2 = (\bm{\tilde{H}}\bm{\tilde{H}}^\top)_{rr}+\lambda_2.
\end{aligned}
\end{equation*}

Since $g^{\bm{\tilde{W}}}_{ir}(s)$ is quadratic in $s$, its Taylor expansion only contains terms up to and including $s^2$:
\begin{equation*}
g^{\bm{\tilde{W}}}_{ir}(s) = g^{\bm{\tilde{W}}}_{ir}(0) + (g^{\bm{\tilde{W}}}_{ir})^\prime(0)\cdot s + \frac{1}{2}(g^{\bm{\tilde{W}}}_{ir})^{\prime\prime}(0)\cdot s^2
\end{equation*}
which is minimized at 
\begin{equation*}
s^* = -\frac{(g^{\bm{\tilde{W}}}_{ir})^\prime(0)}{(g^{\bm{\tilde{W}}}_{ir})^{\prime\prime}(0)}= \frac{(\bm{V} \bm{\tilde{H}}^\top - \bm{\tilde{W}}\bm{\tilde{H}}\bm{\tilde{H}}^\top)_{ir}-\lambda_1-\lambda_2(\bm{\tilde{W}})_{ir}}{(\bm{\tilde{H}}\bm{\tilde{H}}^\top)_{rr}+\lambda_2}
\end{equation*}
and, to make sure $(\bm{\tilde{W}})_{ir} \in [0,(\bm{\tilde{W}_0})_{ir}]$:
\begin{equation*}
(\bm{\tilde{W}})_{ir}^{t+1}=\max((\bm{\tilde{W}_0})_{ir},\min(0,(\bm{\tilde{W}})_{ir}^t+s^*)).
\end{equation*}

\subsection{Update Rule for \texorpdfstring{$\bm{\tilde{M}}$}{M}}
For $\bm{M}\in (\R^+)^{d_p\times d_r}$ the corresponding equations become 
\begin{equation*}
\begin{aligned}
&\min_{\bm{\tilde{M}}:\ 0\leq (\bm{\tilde{M}})_{pr} \leq (\bm{\tilde{M}_0})_{pr}}\frac{1}{2} ||\bm{V}-\bm{\tilde{W}}\bm{\tilde{M}}\bm{\tilde{H}}||_F^2 + \sum_{p,r}\left(\lambda_1(\bm{\tilde{M}})_{pr} + \frac{\lambda_2}{2}(\bm{\tilde{M}})_{pr}^2\right)
\end{aligned}
\end{equation*}

\begin{equation*}
\end{equation*}
\begin{equation*}
\begin{aligned}
g^{\bm{\tilde{M}}}_{pr}(s):=&\frac{1}{2}\sum_{i,j}((\bm{V})_{ij}-(\bm{\tilde{W}}(\bm{\tilde{M}}+s\bm{E_{pr}})\bm{\tilde{H}})_{ij})^2+\lambda_1((\bm{\tilde{M}})_{pr}+s) + \frac{\lambda_2}{2}((\bm{\tilde{M}})_{pr}+s)^2\\
=&\frac{1}{2}\sum_{i,j}((\bm{V})_{ij}-(\bm{\tilde{W}}\bm{\tilde{M}}\bm{\tilde{H}})_{ij}-s(\bm{\tilde{W}})_{ip}(\bm{\tilde{H}})_{rj})^2+\lambda_1((\bm{\tilde{M}})_{pr}+s)\\
&+ \frac{\lambda_2}{2}((\bm{\tilde{M}})_{pr}+s)^2\\
(g^{\bm{\tilde{M}}}_{pr})^\prime(s)=&\sum_{i,j}(-(\bm{V})_{ij}(\bm{\tilde{W}})_{ip}(\bm{\tilde{H}})_{rj}+(\bm{\tilde{W}}\bm{\tilde{M}}\bm{\tilde{H}})_{ij}(\bm{\tilde{W}})_{ip}(\bm{\tilde{H}})_{rj}+2s(\bm{\tilde{W}})_{ip}^2(\bm{\tilde{H}})_{rj}^2)\\
&+\lambda_1 + \lambda_2((\bm{\tilde{M}})_{pr}+s)\\
(g^{\bm{\tilde{M}}}_{pr})^{\prime\prime}(s)=&\sum_{i,j}(2(\bm{\tilde{W}})_{ip}(\bm{\tilde{H}})_{rj}^2)+\lambda_2\\
(g^{\bm{\tilde{M}}}_{pr})^\prime(0)=&\sum_{i,j}(-(\bm{V})_{ij}(\bm{\tilde{W}})_{ip}(\bm{\tilde{H}})_{rj}+(\bm{\tilde{W}}\bm{\tilde{M}}\bm{\tilde{H}})_{ij}(\bm{\tilde{W}})_{ip}(\bm{\tilde{H}})_{rj})+\lambda_1+\lambda_2\bm{(\tilde{M}})_{pr}\\
=&(-\bm{\tilde{W}}^\top \bm{V}\bm{\tilde{H}}^\top+\bm{\tilde{W}}^\top \bm{\tilde{W}}\bm{\tilde{M}}\bm{\tilde{H}}\bm{\tilde{H}}^\top)_{pr}+\lambda_1+\lambda_2(\bm{\tilde{M}})_{pr}\\
(g^{\bm{\tilde{M}}}_{pr})^{\prime\prime}(0)=&\sum_{i}(\bm{\tilde{W}}^\top)_{pi}(\bm{\tilde{W}})_{ip}\sum_j(\bm{\tilde{H}})_{rj}(\bm{\tilde{H}}^\top)_{jr}+\lambda_2 = (\bm{\tilde{W}}^\top \bm{\tilde{W}})_{pp}(\bm{\tilde{H}}\bm{\tilde{H}}^\top)_{rr}+\lambda_2\\
s^* =& \frac{(\bm{\tilde{W}}^\top \bm{V} \bm{\tilde{H}}^\top - \bm{\tilde{W}}^\top \bm{\tilde{W}}\bm{\tilde{M}}\bm{\tilde{H}}\bm{\tilde{H}}^\top)_{pr}-\lambda_1-\lambda_2(\bm{\tilde{M}})_{pr}}{(\bm{\tilde{W}}^\top \bm{\tilde{W}})_{pp}(\bm{\tilde{H}}\bm{\tilde{H}}^\top)_{rr}+\lambda_2}\\
(\bm{\tilde{M}})_{pr}^{t+1}=&\max((\bm{\tilde{M}_0})_{pr},\min(0,((\bm{\tilde{M}})_{ir})^t+s^*)).
\end{aligned}
\end{equation*}

\clearpage
\bibliography{refs}
\bibliographystyle{apalike}

\end{document}